\documentclass[letterpaper,10pt,jounal,twoside]{IEEEtran}

\usepackage[utf8]{inputenc}
\usepackage{cite}
\usepackage{qtree}
\usepackage{tikz}
\usepackage{fontawesome} 
\usepackage{amsmath,amssymb,amsfonts,amsthm}
\usepackage{dsfont,steinmetz,algpseudocode,algorithm}
\usepackage{graphicx,textcomp,algpseudocode,mathtools,nicefrac,txfonts}
\usepackage{cleveref}
\usetikzlibrary{positioning,spy,calc,fit}
\usepackage{subfig}
\usepackage{epstopdf}

\newcommand{\R}{\mathbb{R}}

\newcommand{\N}{\mathcal{N}}
\newcommand{\Nlf}{\mathcal{N}_{\text{leaf}}}
\newcommand{\Nint}{\mathcal{N}_{\text{int}}}

\newcommand{\X}{\mathcal{X}}

\newcommand{\tp}{^{\mbox\tiny\mathsf{T}}}

\newcommand{\Frob}{\rm F}

\newcommand{\T}{\mathcal{T}}

\newcommand{\Q}{\mathcal{Q}}

\def\beq{\begin{equation*}}
\def\eeq{\end{equation*}}
\def\bql{\begin{equation}}
\def\eql{\end{equation}}
\def\bqn{\begin{eqnarray*}}
\def\eqn{\end{eqnarray*}}
\def\bnl{\begin{eqnarray}}
\def\enl{\end{eqnarray}}
\def\bna{\bql\begin{array}{rcl}}
\def\ena{\end{array}\eql}
\def\bnn{\beq\begin{array}{rcl}}
\def\enn{\end{array}\eeq}
\def\bma{\begin{bmatrix}}
\def\ema{\end{bmatrix}}
\def\bmx{\begin{matrix}}
\def\emx{\end{matrix}}
\def\ben{\begin{enumerate}}
\def\een{\end{enumerate}}
\def\bit{\begin{itemize}}
\def\eit{\end{itemize}}
\def\bei{\begin{itemize}}
\def\eei{\end{itemize}}
\def\bet{\begin{tabular}}
\def\eet{\end{tabular}}

\renewcommand{\Re}{\mathbb{R}}

\newtheorem{proposition}{Proposition}

\title{Communication-aware Hierarchical Map Compression of Time-Varying Environments for Mobile Robots}
\author{Daniel T. Larsson and Dipankar Maity 
\thanks{
D. Larsson is an Assistant Professor in the Aerospace and Mechanical Engineering Department at the University of Arizona in Tucson, AZ, USA. dlarsson@arizona.edu
}
\thanks{D. Maity is with the Department of Electrical and Computer Engineering and an affiliated faculty of the North Carolina Battery Complexity, Autonomous Vehicle, and Electrification Research Center (BATT CAVE), University of North Carolina at Charlotte,  NC, 28223, USA. { dmaity@charlotte.edu}}
}

\begin{document}

\maketitle
\thispagestyle{empty}

\begin{abstract}
In this paper, we develop a systematic framework for the time-sequential compression of dynamic probabilistic occupancy grids.
Our approach leverages ideas from signal compression theory to formulate an optimization problem that searches for a multi-resolution hierarchical encoder that balances the quality of the compressed map (distortion) with its description size, the latter of which relates to the bandwidth required to reliably transmit the map to other agents or to store map estimates in on-board memory.
The resulting optimization problem allows for multi-resolution map compressions to be obtained that satisfy available communication or memory resources, and does not require knowledge of the occupancy map dynamics. 
We develop an algorithm to solve our problem, and demonstrate the utility of the proposed framework in simulation on both static (i.e., non-time varying) and dynamic (time-varying) occupancy maps.
\end{abstract}

\section{Introduction}

\IEEEPARstart{R}{ecent} times have witnessed an explosion of capable and intelligent autonomous systems.
These systems have transformed applications such as supply-chain management, search and rescue, and multi-agent mapping and exploration.
Many of the aforementioned applications are performed by a coalition of heterogeneous autonomous systems with various on-board information-processing resources (memory, sensors, hardware, etc.) that share data.
Moreover, the diverse set of applications has given rise to the development of a wide range of system architectures, each with its own set of on-board resources (e.g., sensors, hardware, etc.).
As a result, there is a need for frameworks that allow inter-agent information exchange that takes into account system resources.

One application of particular relevance to our work is that of multi-agent exploration, where a team of robots explore an unknown environment together while sharing information for a common task.
Although multi-agent exploration may take many forms, let us consider the multi-agent exploration setting where two autonomous agents are collaboratively attempting to navigate an unknown environment, akin to the scenario considered in~\cite{psomiadis2024communication}.
In this setting, which contains both exploration and planning problems, two agents, one termed a ``supporter" (e.g., aerial or ground vehicle) works in together with a ground-based agent, termed a ``seeker",  to ensure that the ground-based vehicle safely navigates to its goal position. 
In this scenario, the supporter has the task of exploring the unknown portions of the environment for the other, and to transmit its findings via wireless communication means.
The navigating agent then incorporates the information it receives (e.g., quantized occupancy probabilities of an occupancy grid) into its own world representation, which is subsequently employed to update its navigation (e.g., path) plan.
It is important to note, however, that the wireless connection between the two agents in the real-world is capacity-limited, and thus agents must be selective regarding the (occupancy) information they choose to transmit. 
Motivated by such scenarios, it is our goal in this paper to develop~a map compression framework that compresses probabilistic occupancy grid structures in a systematic manner and allows for communication constraints to be imposed in~a~rigorous~way.

The compression of probabilistic occupancy maps has been studied for some time within the autonomous systems community. 
Given the number of studies that have considered the generation of abstractions, or map compressions, for autonomous systems, it comes as no surprise that existing literature contains a wide breath of frameworks that each vary in scope and objective.
The variations in existing work lie primarily in their assumptions regarding the map encoder, or rather, the structure the framework places on the type of map compression achievable.
For example, in~\cite{Cowlagi2012,Cowlagi2008,Cowlagi2007}, the authors employ wavelet decompositions of occupancy grid maps to generate multi-scale/multi-resolution representations of the operating environment.
The wavelet approach is powerful in that it places minimal assumptions on the structure of the multi-resolution abstractions/compressions, but a notable drawback is that it relies on ad-hoc rules to select the wavelet coefficients, which are oftentimes chosen to be functions that depend on the location of the autonomous system.
As a result, wavelet-based frameworks do not generate map representations that are guaranteed to satisfy communication constraints imposed on the system.
In a related line of work, the authors of~\cite{kraetzschmar2004probabilistic} considered developing occupancy-map compressions by introducing the probabilistic quadtrees (PQTs) framework, which forms occupancy-grid compressions via multi-resolution hierarchical quatree representations.
The benefit of employing the PQTs framework is that it is simple to implement and represents map information in a compact (hierarchical) data structure, but, much like the wavelet-based approaches, pruning is done in a makeshift manner by, for example, removing nodes in the tree whose children have similar occupancy values, as measured by their variance.
Therefore, while PQT tree pruning can be done quickly, the pruning rules leave system designers to trial and error to determine pruning thresholds that allow their intelligent systems to function within operating constraints---a scenario significantly exacerbated when occupancy information changes as a function of time.
Despite their drawbacks, PQT and wavelet-based methods have been used as part of frameworks that aim to reduce the complexity of graph-search algorithms via grid compression~\cite{hauer2015multi,hauer2016reduced,Tsiotras2007,tsiotras_multiresolution_2012,bakolas2008multiresolution}.
More recently, frameworks for map compression that do not rely on user-provided rules have been developed.
In~\cite{larsson2020q}, the authors employ ideas from information-theoretic signal compression to formulate an optimization problem whose solution is a multi-resolution probabilistic (quad)tree.
In contrast to wavelet and the PQT method, the approach in~\cite{larsson2020q}, inspired by the information-bottleneck (IB) method~\cite{tishby2000IBmethod}, creates multi-resolution representations that arise from an optimization problem that maximizes information retention while reducing the size of the representation.
Since the framework in~\cite{larsson2020q} is rooted in information theory, the approach has explicit connections with channel capacity~\cite{thomas2006elements} and is thus readily amenable for the creation of maps that meet communication-constraints. 
Extensions of~\cite{larsson2020q} include~\cite{larsson2021information} and~\cite{larsson2023linear} where memory (or~communication) constraints are imposed as hard-constraints, and~\cite{larsson2022generalized} where a framework for compressing maps that contain semantic information is presented.
Information-theoretic approaches for map compression have also been employed to reduce the computational complexity of graph-based planning and exploration for autonomous systems~\cite{Nelson2015,larsson2021information_b}.
It is important to note that the aforementioned works do not consider map-compression in the case of time-varying map information, nor do they consider what a receiver of compressed map information should do.
To this end, the authors of~\cite{psomiadis2024communication,psomiadis2024multi} consider a two agent \textit{seeker-supporter problem}, where one agent (the supporter) is tasked with assisting another (the seeker) with navigating in an unknown environment.
In their setting, the seeker receives quantized environment information from a mobile supporter, which the seeker then incorporates with its own perception/sensing of the environment to generate an updated estimate of the world. 
Although these works do consider the decoding problem, the encoder employed by the supporter is ``templated" in that the supporter may only choose a map compression structure from a library of possible alternatives, provided by the system designer beforehand, when deciding how to create the map abstraction.
Consequently, while~\cite{psomiadis2024communication,psomiadis2024multi} enable single-agent path planning with information stemming from a secondary source, the compression problem considered is rather limited in flexibility.
In contrast, wavelet, PQT, and the information theoretic approaches discussed, do not~suffer~from~this~limitation.
It is clear from surveying existing literature that a significant challenge is to formulate a problem whose solution can be tractably found via the design of a suitable algorithm.
To this end, the contribution of this paper is the development of a framework that considers both encoding \emph{and} decoding aspects of \emph{time-varying} probabilistic occupancy grid compression where communication constraints may be rigorously enforced.
To achieve our goal, we impose that our encoder correspond to that of a multi-resolution hierarchical tree (e.g., quadtree), which requires minimal information from system designers.
By enforcing a hierarchical encoder structure, we strike a balance between complexity and diversity of expression of the map compressions achievable. 
The remainder of the paper is organized as follows.
In Section~\ref{sec:probFormul} we introduce the general resource-constrained probabilistic map-compression problem, which does not impose any structure on the encoder or decoder.
Then, in Section~\ref{sec:HierarchMapCompProblem}, we present an approach to approximate the general problem in Section~\ref{sec:probFormul} by leveraging multi-resolution hierarchical tree map encoders.
Section~\ref{sec:solnApproach} considers the development of an algorithm, which leverages integer (linear) programming, to solve the problem in Section~\ref{sec:HierarchMapCompProblem}.
The utility of our approach to compress and decode both static (non-time-varying) and dynamic (time-varying) probabilistic occupancy grids is shown in Section~\ref{sec:ResultsDiscuss}, before concluding remarks in Section~\ref{sec:conclusion}. 
%

\section{Problem Formulation}\label{sec:probFormul}

To formalize our problem, we assume that the operating environment ($\X$) at time $k\in\mathbb{N}$ is represented by an $n\times n$ dimensional probabilistic occupancy grid $X_k \in [0,1]^{n \times n}$.
We make no assumptions about the distribution of $X_k$ or its temporal evolution, and assume the former is unknown for all time $k$.
Note that, in light of our assumptions, the distribution over $X_k$ (i.e., the environment) may evolve in an arbitrary fashion due to, for example, dynamic obstacles (humans, robots, cars, etc.).  
To illustrate in more quantifiable terms the need for map compression, we return to the seeker-supporter multi-agent scenario previously discussed, where an agent is to communicate its occupancy grid representation with another autonomous vehicle.
If we assume that each occupancy cell $[X_k]_{ij}$, $i,j\in\{1,\ldots,n\}$, of the map requires $b$ bits to transmit (i.e., to represent the occupancy information, location and size), then transmitting the entire map requires $b n^2$ number of bits.
For large environments where $n\gg 1$, such a transmission will require a significant amount of communication bandwidth.\footnote{For example, an unmanned air vehicle (UAV) sending an occupancy grid of size $1024 \times 1024$ at a rate of 30fps with $b = 8$ requires 240 MB/sec data-rate.}
Consequently, simply reducing the number of bits $b$ for representing the occupancy values---which is what typically done in communication and information theory---is not a viable solution as the required bandwidth will be dominated by the number of cells (i.e., $n^2$). 
Therefore, the problem we consider will aim to efficiently compress the occupancy grid by reducing the number of cells in the representation in a systematic manner while considering the inherent temporal evolution of environment maps in robotic systems applications.
This approach distinguishes our problem and the resulting solution from a typical compression problem in communication theory.
To pose a time-sequential map compression problem that allows us to trade map size and map quality, we will appeal to ideas from traditional signal compression theory.
To this end, we will pose our problem by formulating an optimization problem whose solution is a compressed occupancy grid satisfying the imposed (communication) constraints.
To do so, we must specify how we will quantify both the quality of the compressed representation, generally done with a distortion measure that is to be minimized, and the size of the resulting compressed map.  
With these considerations in mind, the time-sequential map compression problem may be formulated as
\begin{equation}\label{eq:mainOrigObjectiveFunc}
    \min_{f_{e,k},f_{d,k}} ~\lVert X_k - f_{d,k}(f_{e,k}(X_k))\rVert_{\Frob},
\end{equation}
\begin{align}
\text{subject to} \hspace{2.5cm}
	B(f_{e,k}(X_k)) &\leq b_k, \hspace{1.5cm} \label{eq:mainOrigConstraint1}\\
	0 \leq f_{d,k}(f_{e,k}(X_k)) &\leq 1,  \label{eq:mainOrigConstraint2}
\end{align}
where $X_k$ is the observed (probabilistic) occupancy grid, $f_{e,k}$ and $f_{d,k}$ are the encoding and decoding functions, respectively, $B(\cdot)$ is a function that quantifies the bandwidth required to transmit the encoded map $f_{e,k}(X_k)$, and $b_k$ is the available bandwidth at time $k$.
Here $\|\cdot\|_{\Frob}$ denotes the Frobenius norm, and~\eqref{eq:mainOrigConstraint2} is enforced component-wise.
In light of our previous comments on signal compression theory, the problem~\eqref{eq:mainOrigObjectiveFunc}-\eqref{eq:mainOrigConstraint2} measures the quality of the compressed map by the Frobenius norm between the (uncompressed) original and quantized maps in~\eqref{eq:mainOrigObjectiveFunc} (which we wish to minimize), while the size of the representation is assigned according to the function $B(\cdot)$ in~\eqref{eq:mainOrigConstraint1}.

The map-compression problem posed by~\eqref{eq:mainOrigObjectiveFunc}-\eqref{eq:mainOrigConstraint2} attempts to find an encoder-decoder pair $(f_{e,k},f_{d,k})$ at every time $k$ so that the distortion between the original map $X_k$ and the reconstructed map at the receiver $f_{d,k}(f_{e,k}(X_k))$, as measured by $\lVert X_k - f_{d,k}(f_{e,k}(X_k))\rVert_{\Frob}$, is minimized while obeying the communication constraint imposed by~\eqref{eq:mainOrigConstraint1}.
Note that~\eqref{eq:mainOrigConstraint2} enforces the condition that the receiver-reconstructed map corresponds to a valid occupancy grid representation of the operating environment.
Observe that the objective of~\eqref{eq:mainOrigObjectiveFunc}--\eqref{eq:mainOrigConstraint2} is to minimize the \textit{true} distortion based on the realization of map $X_k$ available only at the encoder. 
This is in contrast to traditional encoding and decoding problems that often involve optimizing an \emph{expected} distortion function.
Critically, such a traditional approach would be futile in our setting since neither the encoder nor the decoder has access to the probability distribution of $X_k$ due to the arbitrary and unknown time-evolution of the map. 
As a result, standard encoding-decoding schemes from communication and information theory are not applicable to our problem. 
Lastly, note that the search space in~\eqref{eq:mainOrigObjectiveFunc} is that of the joint space of all functions $(f_{e,k},f_{d,k})$ satisfying~\eqref{eq:mainOrigConstraint1} and~\eqref{eq:mainOrigConstraint2}, which is prohibitively large and introduces a coupling between the encoder and decoder that renders the problem intractable to solve both analytically and numerically.
%

\section{Hierarchical Time-Sequential Map-Compression Problem}\label{sec:HierarchMapCompProblem}

The main difficulty in solving~\eqref{eq:mainOrigObjectiveFunc}-\eqref{eq:mainOrigConstraint2} is attributable to: (i) the coupling between map encoder and decoder, and (ii) the vast search-space for the encoding and decoding functions, which require a suitable parametrization to solve in practice.
Our approach thus proceeds in two phases. 
The first is concerned with developing a (upper) bound and re-parameterization of~\eqref{eq:mainOrigObjectiveFunc}-\eqref{eq:mainOrigConstraint2} which elucidates the encoder-decoder coupling. 
The second phase considers enforcing a hierarchical structure (parameterization) on the map encoder to reduce the feasible search space~of the encoding function. 
At the conclusion of this section we state our hierarchical time-sequential map-compression problem that combines the two elements above and which we will show allows for~an approximate solution of~\eqref{eq:mainOrigObjectiveFunc}-\eqref{eq:mainOrigConstraint2} to be tractably obtained.

\subsection{Re-parameterization of Time-Sequential Map-Compression } \label{subsec:reparamTimeSeqProb}

Before discussing a hierarchical parameterization of the encoder, we have the following result, which provides an upper bound to the objective value to the problem~\eqref{eq:mainOrigObjectiveFunc}--\eqref{eq:mainOrigConstraint2}.
%
\begin{proposition}\label{prop:encDecReparam}
	An upper bound to problem~\eqref{eq:mainOrigObjectiveFunc} subject to~\eqref{eq:mainOrigConstraint1} and~\eqref{eq:mainOrigConstraint2} is obtained from the following optimization:
	\begin{equation}\label{eq:objTScompProb}
		\min_{g_{e,k}, g_{d,k}} \lVert \xi_k - g_{e,k}(\xi_k) \rVert_{\Frob} + \lVert g_{e,k}(\xi_k) - g_{d,k}(g_{e,k}(\xi_k))\rVert_{\Frob},
	\end{equation}
	\begin{align*}
    \mathrm{subject~to} \hspace{3cm}
		B(g_{e,k}(\xi_k)) \leq b_k, \hspace{3cm}\\
		0 \leq \hat X_{k-1} + g_{d,k}(g_{e,k}(\xi_k)) \leq 1,\hspace{3cm}
	\end{align*}
	for every time $k$, where $\xi_k := X_k - \hat X_{k-1}$ is the \emph{map innovation signal}, $\hat X_{k-1}$ is the previous map estimate, $g_{e,k}$ is an~encoding function that compresses the map-innovation $\xi_k$ and~$g_{d,k}$ is a decoding function that uses the compressed map-innovation $g_{e,k}(\xi_k)$ so that, together with $\hat X_{k-1}$, a new map estimate $\hat X_k$ is produced according to $\hat X_{k} = \hat X_{k-1} + g_{d,k}(g_{e,k}(\xi_k))$.
\end{proposition}
\begin{proof}
    The proof is given in the Appendix.
\end{proof}
Notice that Proposition~\ref{prop:encDecReparam} suggests compressing an innovation signal $\xi_k$ as opposed to the map $X_k$ itself -- a widely used technique in the controls and dynamical systems literature~\cite{borkar1997lqg, yuksel2019note, maity2023optimal}. 
It is also worth noting that decision-making based on quantized innovation signals does not necessarily lead to a performance degradation (e.g., predictive coding)~\cite{yuksel2019note}. 
The importance of the optimization problem posed in Proposition~\ref{prop:encDecReparam} lies in two important observations.
Firstly, in re-parameterizing~\eqref{eq:mainOrigObjectiveFunc}-\eqref{eq:mainOrigConstraint2} according to Proposition~\ref{prop:encDecReparam} we have been able to express the problem in terms of the map innovation signal $\xi_k \in  [-1,1]^{n \times n}$, which quantifies the occupancy information discrepancy between the currently observed, finest-resolution, map $X_k$ and the latest map estimate maintained by the receiver $\hat X_{k-1}$.
To see why this is the case, notice that the signal $\xi_k$ is zero over regions of the map where the current estimate held by the receiver $\hat X_{k-1}$ coincides with the occupancy information in the current map $X_k$, and is non-zero elsewhere.
Thus, the map-innovation signal $\xi_k$ quantifies the occupancy information in the current map that is not adequately represented, either due to lack of information, sufficient granularity, or map dynamics, in the receiver's map estimate.
Consequently, compressing $\xi_k$ in place of $X_k$ reduces the possibility of transmitting map information already known to the receiver.
Secondly, the optimization problem posed in Proposition~\ref{prop:encDecReparam} is divided into an encoder loss $\lVert \xi_k - g_{e,k}(\xi_k) \rVert_{\Frob}$, and a decoding loss $\lVert g_{e,k}(\xi_k) - g_{d,k}(g_{e,k}(\xi_k)) \rVert_{\Frob}$.
Importantly, notice that the encoding loss does not depend on the decoding function $g_{d,k}$, and that the encoding and decoding constraints are decoupled, thereby enabling us to separate the encoding and decoding optimization into the problems 
\begin{align} \label{eq:ENCOpt}
\begin{split}
    \text{(ENC)}\hspace{2 cm} & \min_{g_{e,k}} ~\lVert \xi_k - g_{e,k}(\xi_k) \rVert_{\Frob}, \hspace{1.5cm}\\
    \text{subject to}\qquad & B(g_{e,k}(\xi_k)) \leq b_k,    
\end{split}
\end{align}
and 
\begin{align} \label{eq:DECOpt}
\begin{split}
    \text{(DEC)}\hspace{2 cm} & \min_{g_{d,k}} ~\lVert Z_k - g_{d,k}(Z_k) \rVert_{\Frob}, \hspace{1.5cm}\\
    \text{subject to}\qquad & 0 \leq \hat X_{k-1} + g_{d,k}(Z_k) \leq 1,    
\end{split}
\end{align}
where $Z_k$ in~\eqref{eq:DECOpt} is the encoded map innovation (i.e., $g_{e,k}(\xi_k)$) received by the decoder.
Importantly, it can be shown that the optimal decoding function solution to~\eqref{eq:DECOpt}, $g^*_{d,k}$, is given by the \emph{clipping decoder}
\begin{equation}\label{eq:optimalClipDecFunc}
	g^*_{d,k}(Z_k) = \min\{\max\{-\hat X_{k-1},~Z_k\},1-\hat X_{k-1}\},
\end{equation}
where the max and min operations are done element-wise.
The above addresses the first of our challenges pertaining to problem~\eqref{eq:mainOrigObjectiveFunc}-\eqref{eq:mainOrigConstraint2}, namely elucidating the coupling between the map encoder and decoder.
However, we must still specify a suitable parameterization of the encoding function to arrive at a tractable solution algorithm.
To this end, we will enforce that the encoding function $g_{e,k}$ have the structure of a multi-resolution hierarchical tree \cite{larsson2020q} at each time-step $k$.
The connection between multi-resolution trees and signal encoders has been discussed in the literature and interested readers may confer~\cite{larsson2020q,larsson2022generalized} for details, but we here provide a brief review for completeness.
Our development will focus attention on multi-resolution \textit{quadtree} structures, however the results are applicable for any hierarchical structure. 
%

\subsection{Hierarchical Tree Structures as Map Encoders} \label{subsec:treesMapEnc}

To formalize the connection between multi-resolution trees and our problem, we assume that the environment $\X$ is contained within a square grid of side length $2^\ell$ for some integer $\ell > 0$.
Associated with the environment $\X$ is a finest-resolution quadtree $\T_\X$ whose leaf nodes correspond to the cells that comprise $\X$ (i.e., the finest-resolution grid cells), as shown in Fig.~\ref{fig:treeGrid}.
Recall that the collection of occupancy values of the finest-resolution cells at time $k$ is denoted by $X_k \in \R^{2^\ell \times 2^\ell}$, where $[X_k]_{ij} \in [0,1]$ is the occupancy value of the $(i,j)$-th cell at time $k$.
Since each leaf node of $\T_\X$ corresponds to a unique cell index $(i,j)$, $i,j\in\{1,\ldots,2^{\ell}\}$, we define the function $\texttt{ind}(x)$ to return the coordinate $(i,j)$ of the finest resolution leaf node/cell $x\in \Nlf(\T_{\X})$, where $\Nlf(\cdot)$ is the set of leaf nodes for a given tree.
For instance, with reference to Fig.~\ref{fig:treeGrid}, $\texttt{ind}(x_1) = (1,1)$ and $\texttt{ind}(x_2) = (1,2)$, and so on. 
By pruning the finest-resolution tree $\T_\X$ we obtain other, possibly multi-resolution, hierarchical tree representations of $\X$ containing fewer leaf nodes than that of $\T_\X$ (and consequently, fewer cells than $\X$).
We will represent the space of all valid quadtree representations of $\X$ by the set $\T^\Q$.
Then, for any tree $\T\in\T^\Q$, let $\N(\T)$ denote the set of all nodes, and for any $t\in\N(\T_\X)$, we take $\Nlf(\T_{\X(t)})$ to be the set of finest-resolution nodes that are descendant from the node $t$.
See Fig.~\ref{fig:treeGrid} for illustration.

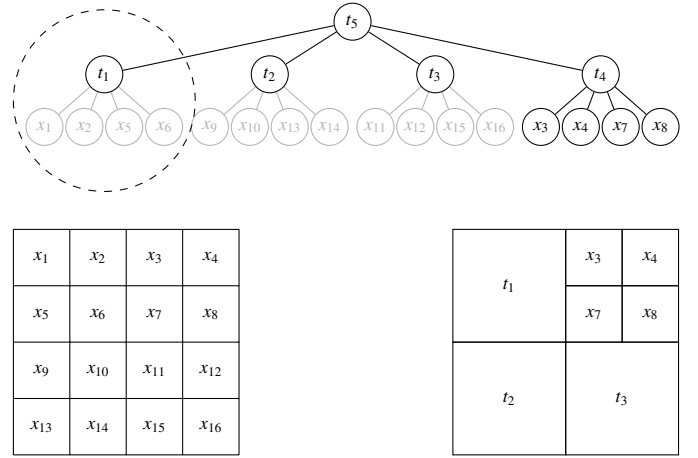
\begin{figure}
    \begin{tikzpicture}
      [level distance=7mm,
       level 1/.style={sibling distance=22mm},
       level 2/.style={sibling distance=5.3mm},spy using outlines={rectangle, red, magnification=1.2,size=25mm,connect spies}]
      \node[draw,black,fill=white,circle,inner sep=1pt,minimum size=14pt] {\scriptsize $t_5$}
         child {node[draw,black,fill=white,circle,inner sep=1pt,minimum size=14pt] (t_1) {\scriptsize $t_1$}
           child {node[draw,black,opacity=0.3,fill=white,circle,inner sep=1pt,minimum size=14pt] (x_1) {\scriptsize $x_1$} edge from parent[draw,black,opacity=0.3]}
           child {node[draw,black,opacity=0.3,fill=white,circle,inner sep=1pt,minimum size=14pt] (x_2) {\scriptsize $x_2$} edge from parent[draw,black,opacity=0.3]}
           child {node[draw,black,opacity=0.3,fill=white,circle,inner sep=1pt,minimum size=14pt] (x_5) {\scriptsize $x_5$} edge from parent[draw,black,opacity=0.3]}
           child {node[draw,black,opacity=0.3,fill=white,circle,inner sep=1pt,minimum size=14pt] (x_6) {\scriptsize $x_6$} edge from parent[draw,black,opacity=0.3]}
         }
         child {node[draw,black,fill=white,circle,inner sep=1pt,minimum size=14pt] {\scriptsize $t_2$}
           child {node[draw,black,opacity=0.3,fill=white,circle,inner sep=1pt,minimum size=14pt] {\scriptsize $x_{9}$} edge from parent[draw,black,opacity=0.3]}
           child {node[draw,black,opacity=0.3,fill=white,circle,inner sep=1pt,minimum size=14pt] {\scriptsize $x_{10}$} edge from parent[draw,black,opacity=0.3]}
           child {node[draw,black,opacity=0.3,fill=white,circle,inner sep=1pt,minimum size=14pt] {\scriptsize $x_{13}$} edge from parent[draw,black,opacity=0.3]}
           child {node[draw,black,opacity=0.3,fill=white,circle,inner sep=1pt,minimum size=14pt] {\scriptsize $x_{14}$} edge from parent[draw,black,opacity=0.3]}
         }
         child {node[draw,black,fill=white,circle,inner sep=1pt,minimum size=14pt] {\scriptsize $t_3$}
           child {node[draw,black,opacity=0.3,fill=white,circle,inner sep=1pt,minimum size=14pt] {\scriptsize $x_{11}$} edge from parent[draw,black,opacity=0.3]}
           child {node[draw,black,opacity=0.3,fill=white,circle,inner sep=1pt,minimum size=14pt] {\scriptsize $x_{12}$} edge from parent[draw,black,opacity=0.3]}
           child {node[draw,black,opacity=0.3,fill=white,circle,inner sep=1pt,minimum size=14pt] {\scriptsize $x_{15}$} edge from parent[draw,black,opacity=0.3]}
           child {node[draw,black,opacity=0.3,fill=white,circle,inner sep=1pt,minimum size=14pt] {\scriptsize $x_{16}$} edge from parent[draw,black,opacity=0.3]}
         }
         child {node[draw,black,fill=white,circle,inner sep=1pt,minimum size=14pt] {\scriptsize $t_4$}
           child {node[draw,black,fill=white,circle,inner sep=1pt,minimum size=14pt] {\scriptsize $x_3$}}
           child {node[draw,black,fill=white,circle,inner sep=1pt,minimum size=14pt] {\scriptsize $x_4$}}
           child {node[draw,black,fill=white,circle,inner sep=1pt,minimum size=14pt] {\scriptsize $x_7$}}
           child {node[draw,black,fill=white,circle,inner sep=1pt,minimum size=14pt] {\scriptsize $x_8$}}
         };

        \node [draw,dashed,inner sep=0pt, circle,fit={(t_1)(x_6)(x_1)}] {};
        
    \end{tikzpicture}
    \\
    \vspace{2pt}
    \\
    \begin{tikzpicture}[scale = 0.75]
        \draw[step=1.0,black,thin] (0,0) grid (4,4);

        \node at (3.5,3.5) {\scriptsize $x_{4}$}; 
        \node at (2.5,3.5) {\scriptsize $x_{3}$}; 
        \node at (3.5,2.5) {\scriptsize $x_{8}$}; 
        \node at (2.5,2.5) {\scriptsize $x_{7}$}; 

        \node at (1.5,3.5) {\scriptsize $x_{2}$}; 
        \node at (0.5,3.5) {\scriptsize $x_{1}$}; 
        \node at (1.5,2.5) {\scriptsize $x_{6}$}; 
        \node at (0.5,2.5) {\scriptsize $x_{5}$}; 

        \node at (3.5,1.5) {\scriptsize $x_{12}$}; 
        \node at (2.5,1.5) {\scriptsize $x_{11}$}; 
        \node at (3.5,0.5) {\scriptsize $x_{16}$}; 
        \node at (2.5,0.5) {\scriptsize $x_{15}$}; 

        \node at (1.5,1.5) {\scriptsize $x_{10}$}; 
        \node at (0.5,1.5) {\scriptsize $x_{9}$}; 
        \node at (1.5,0.5) {\scriptsize $x_{14}$}; 
        \node at (0.5,0.5) {\scriptsize $x_{13}$}; 
    \end{tikzpicture}
    \hfill
    \begin{tikzpicture}[scale = 0.75]
        \draw[black,thin] (0,0) rectangle (2,2);
        \draw[black,thin] (0,2) rectangle (2,4);
        \draw[black,thin] (2,0) rectangle (4,2);

        \draw[black,thin] (2,2) rectangle (3,3);
        \draw[black,thin] (3,2) rectangle (4,3);
        \draw[black,thin] (2,3) rectangle (3,4);
        \draw[black,thin] (3,3) rectangle (4,4);
        
        \node at (3.5,3.5) {\scriptsize $x_{4}$}; 
        \node at (2.5,3.5) {\scriptsize $x_{3}$}; 
        \node at (3.5,2.5) {\scriptsize $x_{8}$}; 
        \node at (2.5,2.5) {\scriptsize $x_{7}$}; 

        \node at (1,3) {\scriptsize $t_{1}$}; 
        
        \node at (1,1) {\scriptsize $t_{2}$}; 

        \node at (3,1) {\scriptsize $t_{3}$}; 
    \end{tikzpicture}
    \caption{(Bottom left) Finest-resolution $4\times 4$ grid with cells $\{x_1,\ldots,x_{16}\}$. (Bottom right) compressed grid aggregating cells $\{x_1,x_2,x_5,x_6\}$ to node $t_1$, cells $\{x_9,x_{10},x_{13},x_{14}\}$ to $t_2$, and so on. (Top) The quadtree tree $\T$ shown in black that corresponds to the compressed grid at right. The quadtree $\T_\X $ is the tree containing all the finest resolution cells $x_1,\ldots,x_{16}$ as leaf nodes, with the (sub)tree $\T_{\X(t_1)}$ shown by tree contained entirely within the dashed circle. $\T$ has seven leaf nodes (i.e., $\Nlf(\T) = \{t_1,t_2,t_3,x_3,x_4,x_7,x_8\}$) with corresponding regions $S_1^{\T}=\{x_1,x_2,x_5,x_6\},~S_2^{\T}=\{x_9,x_{10},x_{13},x_{14}\}$, and so forth. 
    }
    \label{fig:treeGrid}
\end{figure}

The connection between multi-resolution quadtrees and signal encoders can be more explicitly seen by noting that each quadtree represents a partitioning of the cells of  $\X$ into disjoint regions $S_1^\T,\ldots,S_{\lvert \Nlf(\T)\rvert}^\T$ according to 
\begin{equation}
    S_r^\T = \{x \in \Nlf(\T_{\X(t_r)}) : t_r\in\Nlf(\T)\},
\end{equation}
for all $r = 1,\ldots, \lvert\Nlf(\T)\rvert$.
Notice that for any tree $\T$ and any $r \in \{1,\ldots,\lvert \Nlf(\T) \rvert \}$, we have $S_r^\T \subseteq \Nlf(\T_\X)$.
Using the $\texttt{ind}(\cdot)$ function with a slight abuse of notation, we obtain 
\begin{align} \label{eq:indFunction}
    &\texttt{ind}(S_r^\T) = \{\texttt{ind}(x) : x\in S_r^\T\} .
\end{align}
The importance of the function $\texttt{ind}(\cdot)$ in~\eqref{eq:indFunction} is that it provides a bridge between the index numbering of the leaf nodes of the trees and the grid ordering of the occupancy grid cells in the map $X_k$ representing the world environment at time $k$.
To fully define the mapping between an encoding $g_{e,k}$  and a quadtree, we must, in addition to specifying a method by which cells are aggregated, also specify the occupancy value that will be assigned to represent all nodes within an aggregation in the compressed representation.
In other words, we must specify the reproduction points~\cite{thomas2006elements}.
To this end, we notice that since $g_{e,k}(\xi_k)$ defines an aggregation of the map innovation signal $\xi_k$, it follows that $g_{e,k}(\xi_k)$ is piece-wise constant on the partitions defined by the encoder.
That is, for a given quadtree $\T$, $g_{e,k}(\xi_k)$ takes $ \lvert \Nlf(\T)\rvert$ values $g_{1,k},g_{2,k},\ldots,g_{\lvert \Nlf(\T)\rvert, k}$, $g_{r,k}\in\Re$ for $r = 1,\ldots,\lvert \Nlf(\T)\rvert$ and $k\in\mathbb{N}$, one for each of the partitions $S^{\T}_1,\ldots,S^{\T}_{\lvert \Nlf(\T)\rvert}$ of aggregated finest-resolution cells.
The problem of selecting reproduction points is then one of determining the values $g_{r,k}$.
To this end, given the innovation map $\xi_k$ and quadtree $\T$, we select the reproduction points according to
\begin{equation}\label{eq:reprodPtsOptim}
	g_{r,k} = \frac{1}{|S_r^{\T}|} \sum\nolimits_{(i,j) \in \texttt{ind}(S_r^{\T})} [\xi_k]_{ij}, \quad r = 1,\ldots,\lvert \Nlf(\T)\rvert.
\end{equation}
Note that selecting the reproduction points according to~\eqref{eq:reprodPtsOptim} minimizes $\lVert \xi_k - g_{e,k}(\xi_k)\rVert_F$ over all choices of $g_r$, independently of the tree encoder employed.
This can be more clearly seen by noting that $\lVert \xi_k - g_{e,k}(\xi_k)\rVert_F$ can be written as
\begin{align}\label{eq:encCostTree}
    \lVert \xi_k - g_{e,k}(\xi_k)\rVert_{\Frob} = \left( \sum\nolimits_{r=1}^{\lvert \Nlf(\T) \rvert}\! \sum\nolimits_{(i,j) \in \texttt{ind}(S_r^{\T})} \left([\xi_k]_{ij} - g_{r,k}\right)^2 \right)^{\frac{1}{2}},
\end{align}
which shows that \eqref{eq:reprodPtsOptim} provides the optimal $g_r$'s.
Lastly, we will assume that each occupancy cell requires $b > 0$ bits of information for storage or communication, and thus a tree $\T$ with $\lvert\Nlf(\T)\rvert$ leaf nodes requires $b \lvert\Nlf(\T)\rvert$ bits of memory or bandwidth.
If $b_k$ is the maximum number of bits for transmission at time $k$, then the communication constraint becomes one of restricting the number of leaf nodes according to $\lvert \Nlf(\T) \rvert \leq \lfloor b_k/b\rfloor$, where $\lfloor \cdot \rfloor$ is the flooring function.
Therefore, without loss of generality, from now onward we will rewrite the communication constraint as $\lvert \Nlf(\T) \rvert \leq n_k$. 
%

\subsection{Time-Sequential Hierarchical Map-Compression Problem}

By assembling the results from Sections~\ref{subsec:reparamTimeSeqProb} and~\ref{subsec:treesMapEnc}, we arrive at the time-sequential hierarchical map-compression problem, the algorithmic solution to which will be the focus of the remainder of the paper.
\textbf{Time-sequential hierarchical map-compression problem.} Given the world $\X$, a time-horizon $N>0$, a sequence of (probabilistic) occupancy grids $\{X_k\}_{k=0}^{N}$, and a collection of communication-bandwidth limits $\{n_k\}_{k=0}^{N}$, the \emph{time-sequential hierarchical map-compression problem} is one of generating a sequence of hierarchical tree representations $\{\T_k\}_{k=0}^{N}$ by solving~\eqref{eq:ENCOpt} over the space $\T^\Q$, while decoding according to~\eqref{eq:DECOpt}. 
%

\section{Solution Approach}\label{sec:solnApproach}

While it was shown in the previous section that a solution to the decoding problem~\eqref{eq:DECOpt} can be readily obtained independently of the specific encoder employed, obtaining a solution to~\eqref{eq:ENCOpt} over the space of tree-structured encoders is not obvious.
To this end, while we may straightforwardly employ the decoder~\eqref{eq:optimalClipDecFunc} in our algorithm, we will need to exploit the problem structure to design an encoder for our problem.
It is therefore that we next discuss the encoding scheme of our framework.
%

\subsection{Encoding and Decoding Scheme}

The goal is to select a tree-structured encoder that solves problem~\eqref{eq:ENCOpt} over the space of hierarchical representations $\T^\Q$.
The dependence of~\eqref{eq:ENCOpt} on the tree $\T\in\T^\Q$ can be seen via~\eqref{eq:encCostTree}, since the latter has assumed the encoding function (i.e.,~$g_{e,k}$) correspond to a hierarchical tree. 
Employing the ideas from Section~\ref{sec:HierarchMapCompProblem}, the encoding problem is
\begin{equation}\label{eq:treeEncObj}
\min_{\T_k\in\T^\Q} \sum\nolimits_{r=1}^{\lvert \Nlf(\T_k) \rvert}\sum\nolimits_{(i,j) \in \texttt{ind}(S_r^{\T_k})} \left([\xi_k]_{ij} - g_{r,k}\right)^2,
\end{equation}
subject to
\begin{equation}\label{eq:treeEncCons}
    \lvert \Nlf(\T_k) \rvert \leq n_k.
\end{equation}
Notice that, if problem~\eqref{eq:treeEncObj}-\eqref{eq:treeEncCons} results with a tree-encoder $\T_k\in\T^\Q$ for which the objective~\eqref{eq:treeEncObj} attains the global minimum of zero, then the updated map $\hat X_k = X_k$.
This is because if~\eqref{eq:treeEncObj} is zero, then from~\eqref{eq:encCostTree} it follows that $\xi_k = g_{e,k}(\xi_k) = Z_k$ and so from~\eqref{eq:DECOpt} and~\eqref{eq:optimalClipDecFunc} we have $\hat X_{k} = \hat X_{k-1} + g^*_{d,k}(g_{e,k}(\xi_k)) = \hat X_{k-1} + \xi_k = \hat X_{k-1} + (X_k - \hat X_{k-1}) = X_k$, since $X_k$ is a feasible occupancy grid and thus no clipping is required. 
Our objective is now to devise an approach towards obtaining a solution to~\eqref{eq:treeEncObj}-\eqref{eq:treeEncCons}.
To this end, it should be noted that~\eqref{eq:treeEncObj}-\eqref{eq:treeEncCons} is a discrete optimization problem since the search space $\T^\Q$ is a countable set.
Consequently, obtaining a solution to~\eqref{eq:treeEncObj}-\eqref{eq:treeEncCons} is not easy as grid-search approaches that generate an exhaustive list of candidate solutions are not viable for anything but small environment sizes.
Despite this challenge, we note that we may decompose our problem~\eqref{eq:treeEncObj}-\eqref{eq:treeEncCons} into incremental cost and constraint contributions by each node $t\in \N(\T)$ that is expanded when constructing any $\T\in\T^\Q$ from the the root node by adopting a similar approach to that discussed in~\cite{larsson2021information,larsson2023linear}
To this end, by leveraging the properties of the encoding distortion function~\eqref{eq:encCostTree}, we find that~\eqref{eq:treeEncObj}-\eqref{eq:treeEncCons} can be written as
\begin{equation}\label{eq:encObj}
    \min_{\mathbf{z}}~~\mathbf{z}\tp \Delta_D
\end{equation}
subject to
\begin{align}
     (\theta - 1)\mathbf{z}\tp\mathbf{1}_{\lvert \Nint(\T_\X) \rvert} & \leq n_k - 1, \label{eq:encCons1} \\
     A \mathbf{z} &\leq 0, \label{eq:encCons2} \\
     [\mathbf{z}]_t &\in \{0,1\}, \quad t\in\Nint(\T_\X), \label{eq:encCons3}
\end{align}
where $\Nint(\T_\X) = \N(\T_\X)\setminus\Nlf(\T_\X)$ is the set of interior (i.e., non-leaf) nodes of $\T_\X$, $\theta \in \mathbb N$ is the (assumed constant) branching factor, and $\mathbf{1}_{\lvert \Nint(\T_\X) \rvert}$ is a column vector of ones of length $\lvert \Nint(\T_\X) \rvert$.
Relation~\eqref{eq:encCons1} represents the communication constraint placed on the system, which follows since, for trees with constant branching factor $\theta$, the number of leaf nodes and interior nodes of any tree $\T\in\T^\Q$ can be related according to $\lvert \Nlf(\T) \rvert = (\theta-1)\lvert \Nint(\T) \rvert + 1$, and $\lvert \Nint(\T) \rvert = \mathbf{z}\tp \mathbf{1}_{\lvert \Nint(\T_\X) \rvert}$.
The constraints~\eqref{eq:encCons2} and~\eqref{eq:encCons3} ensure that the integer vector $\mathbf z$ corresponds to some $\T\in\T^\Q$ by means of two observations.
The first is that each interior node (i.e., expandable node) $t\in\Nint(\T_\X)$ of the tree $\T_\X\in\T^\Q$ may be either expanded or aggregated.
Consequently, the vector $\mathbf{z}\in \mathbb{Z}^{\lvert \Nint(\T_\X)\rvert}$ where if $[\mathbf z]_t = 1$ then the node $t\in\Nint(\T_\X)$ is expanded, and $[\mathbf z]_t = 0$ otherwise, leading to~\eqref{eq:encCons3}.
The second observation is that, for a hierarchical tree structure to represent a valid multi-resolution depiction of the environment, it must be the case that a child node $t'$ of some node $t\in\Nint(\T_\X)$ cannot be expanded unless its parent $t$ has been, as illustrated in Fig.~\ref{fig:ILPConstraint}.
Moreover, if a parent $t$ of some node $t'$ has been expanded, this does not imply that $t'$ itself be.
Collecting the above observations, we see that we can express this expansion constraint as $[\mathbf z]_{t'} - [\mathbf z]_t \leq 0$, and write it in the form of~\eqref{eq:encCons2} where the matrix $A$ is such that $[A]_{ij}\in\{-1,0,1\}$.

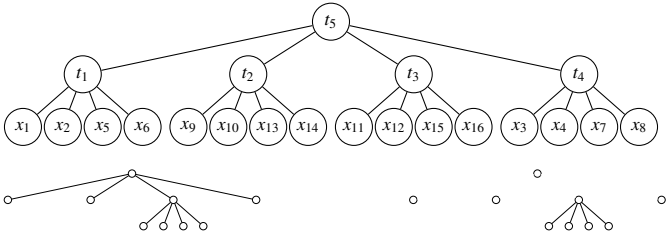
\begin{figure}
    \begin{tikzpicture}
      [level distance=7mm,
       level 1/.style={sibling distance=22mm},
       level 2/.style={sibling distance=5.3mm}]
      \node[draw,black,fill=white,circle,inner sep=1pt,minimum size=14pt] {\scriptsize $t_5$}
         child {node[draw,black,fill=white,circle,inner sep=1pt,minimum size=14pt] (t_1) {\scriptsize $t_1$}
           child {node[draw,black,fill=white,circle,inner sep=1pt,minimum size=14pt] (x_1) {\scriptsize $x_1$} edge from parent[draw,black]}
           child {node[draw,black,fill=white,circle,inner sep=1pt,minimum size=14pt] (x_2) {\scriptsize $x_2$} edge from parent[draw,black]}
           child {node[draw,black,fill=white,circle,inner sep=1pt,minimum size=14pt] (x_5) {\scriptsize $x_5$} edge from parent[draw,black]}
           child {node[draw,black,fill=white,circle,inner sep=1pt,minimum size=14pt] (x_6) {\scriptsize $x_6$} edge from parent[draw,black]}
         }
         child {node[draw,black,fill=white,circle,inner sep=1pt,minimum size=14pt] {\scriptsize $t_2$}
           child {node[draw,black,fill=white,circle,inner sep=1pt,minimum size=14pt] {\scriptsize $x_{9}$} edge from parent[draw,black]}
           child {node[draw,black,fill=white,circle,inner sep=1pt,minimum size=14pt] {\scriptsize $x_{10}$} edge from parent[draw,black]}
           child {node[draw,black,fill=white,circle,inner sep=1pt,minimum size=14pt] {\scriptsize $x_{13}$} edge from parent[draw,black]}
           child {node[draw,black,fill=white,circle,inner sep=1pt,minimum size=14pt] {\scriptsize $x_{14}$} edge from parent[draw,black]}
         }
         child {node[draw,black,fill=white,circle,inner sep=1pt,minimum size=14pt] {\scriptsize $t_3$}
           child {node[draw,black,fill=white,circle,inner sep=1pt,minimum size=14pt] {\scriptsize $x_{11}$} edge from parent[draw,black]}
           child {node[draw,black,fill=white,circle,inner sep=1pt,minimum size=14pt] {\scriptsize $x_{12}$} edge from parent[draw,black]}
           child {node[draw,black,fill=white,circle,inner sep=1pt,minimum size=14pt] {\scriptsize $x_{15}$} edge from parent[draw,black]}
           child {node[draw,black,fill=white,circle,inner sep=1pt,minimum size=14pt] {\scriptsize $x_{16}$} edge from parent[draw,black]}
         }
         child {node[draw,black,fill=white,circle,inner sep=1pt,minimum size=14pt] {\scriptsize $t_4$}
           child {node[draw,black,fill=white,circle,inner sep=1pt,minimum size=14pt] {\scriptsize $x_3$}}
           child {node[draw,black,fill=white,circle,inner sep=1pt,minimum size=14pt] {\scriptsize $x_4$}}
           child {node[draw,black,fill=white,circle,inner sep=1pt,minimum size=14pt] {\scriptsize $x_7$}}
           child {node[draw,black,fill=white,circle,inner sep=1pt,minimum size=14pt] {\scriptsize $x_8$}}
         };
    \end{tikzpicture}
    \\
    \vspace{-3pt}
    \\
    \scalebox{0.5}{%
       \begin{tikzpicture}
        [level distance=7mm,
           level 1/.style={sibling distance=22mm},
           level 2/.style={sibling distance=5.3mm},every node/.style={scale=0.4}]
        \node[draw,black,fill=white,circle,inner sep=1pt,minimum size=14pt]  (feasibleTree) at (0,0) {}
            child {node[draw,black,fill=white,circle,inner sep=1pt,minimum size=14pt] (t_1) {}
            }
            child {node[draw,black,fill=white,circle,inner sep=1pt,minimum size=14pt] {}
            }
            child {node[draw,black,fill=white,circle,inner sep=1pt,minimum size=14pt] {}
                child {node[draw,black,fill=white,circle,inner sep=1pt,minimum size=14pt] {} edge from parent[draw,black]}
                child {node[draw,black,fill=white,circle,inner sep=1pt,minimum size=14pt] {} edge from parent[draw,black]}
                child {node[draw,black,fill=white,circle,inner sep=1pt,minimum size=14pt] {} edge from parent[draw,black]}
                child {node[draw,black,fill=white,circle,inner sep=1pt,minimum size=14pt] {} edge from parent[draw,black]}
            }
            child {node[draw,black,fill=white,circle,inner sep=1pt,minimum size=14pt] {}
        };
     \end{tikzpicture}
    }
    \hfill
    \scalebox{0.5}{%
       \begin{tikzpicture}
        [level distance=7mm,
           level 1/.style={sibling distance=22mm},
           level 2/.style={sibling distance=5.3mm},every node/.style={scale=0.4}]
        \node[draw,black,fill=white,circle,inner sep=1pt,minimum size=14pt]  (feasibleTree) at (0,0) {}
            child {node[draw,black,fill=white,circle,inner sep=1pt,minimum size=14pt] (t_1) {} edge from parent[white]
            }
            child {node[draw,black,fill=white,circle,inner sep=1pt,minimum size=14pt] {} edge from parent[white]
            }
            child {node[draw,black,fill=white,circle,inner sep=1pt,minimum size=14pt] {} edge from parent[white]
                child {node[draw,black,fill=white,circle,inner sep=1pt,minimum size=14pt] {} edge from parent[draw,black]}
                child {node[draw,black,fill=white,circle,inner sep=1pt,minimum size=14pt] {} edge from parent[draw,black]}
                child {node[draw,black,fill=white,circle,inner sep=1pt,minimum size=14pt] {} edge from parent[draw,black]}
                child {node[draw,black,fill=white,circle,inner sep=1pt,minimum size=14pt] {} edge from parent[draw,black]}
            }
            child {node[draw,black,fill=white,circle,inner sep=1pt,minimum size=14pt] {} edge from parent[white]
        };
     \end{tikzpicture}
    }
    \caption{Illustration of integer program constraints. (top) For the tree shown, the interior nodes (i.e., those nodes with children) are $\{t_1,t_2,t_3,t_4,t_5\}$. To represent an allowable pruning of the tree (bottom left), it must be the case that for any of the nodes $t_1,~t_2,~t_3$ or $t_4$ to be expanded, their parent $t_5$ must be expanded as well. (bottom right) An invalid pruning of the tree obtained by not enforcing the parental expansion constraint, with $t_3$ expanded without the expansion of $t_5$. By taking $[\mathbf z]_i$ to be the state (expanded or not) of node $t_i$, we see that for valid pruning we have $[\mathbf z]_i - [\mathbf z]_5 \leq 0$ for each $i = 1,\ldots,4$. Collecting these inequalities leads to~\eqref{eq:encCons2}.}
    \label{fig:ILPConstraint}
\end{figure}

In~\eqref{eq:encObj} the vector $\Delta_D \in \Re^{\lvert \Nint(\T_\X)\rvert}$ provides the incremental reduction in the map distortion~\eqref{eq:treeEncObj} due to expanding nodes to create the tree $\T\in\T^\Q$.
That is, each element $[\Delta_D]_t$ of the vector $\Delta_D$ quantifies the reduction in distortion due to expanding, or adding, the node $t$ to the solution of~\eqref{eq:treeEncObj}, and can be shown to be given by
\begin{equation}\label{eq:deltaDistortion}
    [\Delta_D]_t = 
   	s(t)\tp
    H(t)
    s(t),
\end{equation}
where the vector $s(t) =[s_{t'_1},\ldots,s_{t'_n}]\tp$ contains the reproduction point values $s_{t'_i}$ for each of the children $t'_i$ of $t$, and  $H(t) \in \Re^{\theta \times \theta}$ is a matrix whose entries $[H(t)]_{ij}$ are given by 
\begin{equation}
	[H(t)]_{ij} = 
	 \begin{cases}
		\frac{h(t) - h(t_i)\theta^2}{\theta^2}, & \text{ if } i = j, \\
		\frac{h(t)}{\theta^2}, & \text{ otherwise},
	\end{cases}
\end{equation}
where $h(t) = \lvert \Nlf(\T_{\X(t)}) \rvert$ are the number of finest-resolution nodes descendant from the node $t$.
Recall that the reproduction points in our work are determined according to~\eqref{eq:reprodPtsOptim}, although the relation~\eqref{eq:deltaDistortion} holds irrespective of the specific choice of reproduction points.
Once the solution to~\eqref{eq:encObj}-\eqref{eq:encCons3} is known, the corresponding tree $\T_k\in\T^\Q$ can be identified from the solution vector $\mathbf z$, from which the value of $g_{e,k}(\xi_k)$ can be readily determined as 
\begin{equation}\label{eq:encodingZt}
    [Z_k]_{ij} = g_{r,k}, \quad r\in\{q : (i,j)\in \texttt{ind}(S^{\T_k}_q)\},
\end{equation}
where $g_{r,k}$ is computed according to~\eqref{eq:reprodPtsOptim}, and $Z_k = g_{e,k}(\xi_k)$ is the quantized representation of the map innovation signal at time $k$.
Since the collection $S^{\T_k}_1,\ldots,S^{\T_k}_{\lvert \Nlf(\T_k)\rvert}$ forms a partition of the occupancy grid $X_k$ induced by the tree $\T_k\in\T^\Q$, it follows that any cell $(i,j)$ of $X_k$ is aggregated to exactly one leaf $t_r\in\Nlf(\T_k)$ and thus~\eqref{eq:encodingZt} is a singleton for each $(i,j)$.
The signal $Z_k$ is then transmitted to the receiving agent.
Upon obtaining $Z_k$, the receiving agent will create $\hat X_k$ according to~\eqref{eq:optimalClipDecFunc}; that is
\begin{equation}\label{eq:treeDecGen}
    \hat X_{k} = \hat X_{k-1} + \min\{\max\{-\hat X_{k-1},~Z_k\},1-\hat X_{k-1}\},
\end{equation}
where the minimum and maximum operations are performed element-wise.
We now integrate the encoding and decoding schemes discussed above into an algorithmic framework.
%

\begin{algorithm}[tb] 
	\caption{Multi-resolution occupancy grid compression}
	\begin{algorithmic}[1]
		\State \textbf{Input:} Environment side-length ($\ell$), Time-Horizon length $N > 0$, Collection of leaf-node constraints $\{n_k\}_{k=1}^{N}$ 
		\For{$k = 1:N$} \label{algLine:forLoopTime}
		\State $X_k \gets \texttt{environmentUpdate}()$; \label{algLine:trueMapUpdate}
		\State $\xi_k \gets X_k - \hat{X}_{k-1}$; \label{algLine:generateMapInnovation}
		\State $\T_k \gets \texttt{HierarchicalTreeProblem}(\xi_k, n_k, \ell)$;  \label{algLine:SolveILP}
		\State $Z_k \gets \texttt{GenerateQuantizedInnovation}(\T_k, \xi_k)$; \label{algLine:genQuantizeMapInnovation}
		\State $\hat{X}_{k} \gets \texttt{SendToSender}(Z_k)$; \label{algLine:sendQuanMapInnovation}
		\EndFor
	\end{algorithmic}\label{alg:mapCompressionAlg}
\end{algorithm}

\subsection{Time-Sequential Hierarchical Map-Compression Algorithm}

The time-sequential hierarchical map-compression algorithm is presented in Algorithm~\ref{alg:mapCompressionAlg}.
The algorithm first generates the current occupancy map from environment perceptual information in line~\ref{algLine:trueMapUpdate}. 
The sender then, in line~\ref{algLine:generateMapInnovation}, computes the map innovation signal $\xi_k$.
Using $\xi_k$ and $n_k$, the sender solves~\eqref{eq:treeEncObj}-\eqref{eq:treeEncCons} in line~\ref{algLine:SolveILP} to generate the encoder $\T_k \in \T^\Q$.
The quantized innovation signal $Z_k$ is generated in line~\ref{algLine:genQuantizeMapInnovation} according to~\eqref{eq:encodingZt}, and sent to the receiving agent in line~\ref{algLine:sendQuanMapInnovation}.
Upon receiving $Z_k$, the current map estimate $\hat X_{k-1}$ and~\eqref{eq:treeDecGen} are employed to generate $\hat X_{k}$.
Once $\hat X_k$ is generated in line~\ref{algLine:sendQuanMapInnovation} of Algorithm~\ref{alg:mapCompressionAlg}, the algorithm proceeds to the next time-step and repeats.
Importantly, note that that the number of time-steps $N>0$ and the leaf-node constraints $\{n_k\}_{k=1}^{N}$ need not be provided apriori as the algorithm is agnostic to future information, since the algorithm leverages only current and past information to generate the map encoder.
Consequently, the algorithm is amenable to online decision-making or path-planning applications.
%

\section{Simulation Results and Discussion}\label{sec:ResultsDiscuss}

\begin{figure}[b]
	\centering
	\subfloat[]{\label{fig:originalTrueEnv}\includegraphics[width=0.48\columnwidth]{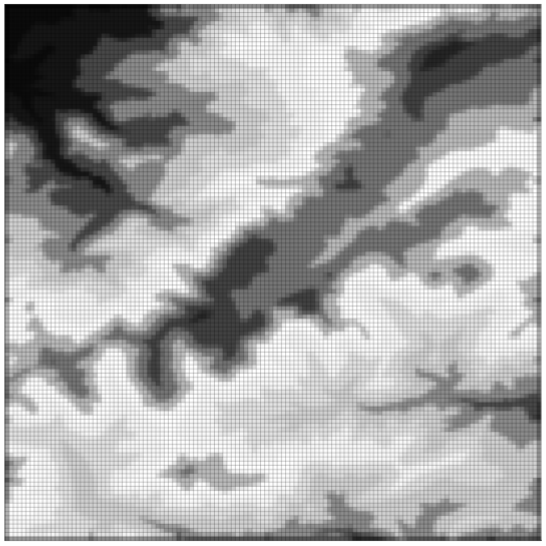}}
    \hfill
    \subfloat[]{\label{fig:fullEnvWithPath}\includegraphics[width=0.48\columnwidth]{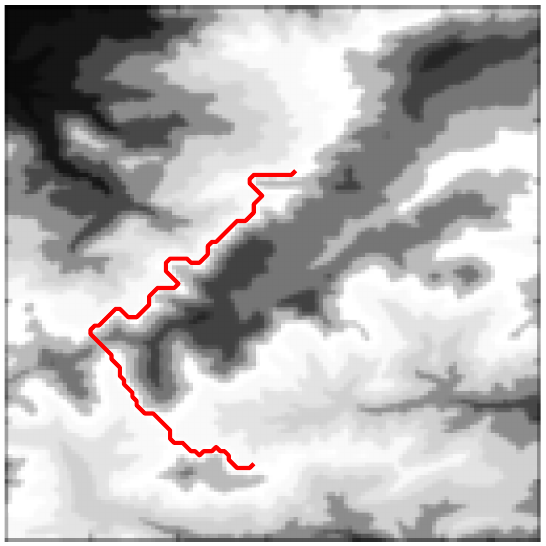}}
    \caption{(a) The 128$\times$128 probabilistic environment. (b) Full-resolution 128$\times$128 probabilistic environment together with amoeba path for time-varying map compression. Grey scale intensity corresponds to probability that a cell is occupied.}
    \label{fig:FRenvCaseIandII}
\end{figure}
%

%
To evaluate our approach, we consider two scenarios: in the first, the environment is static (i.e., $X_k = X$, for all time $k$), while in the second, we compress a time-varying map where the occupancy values change due to the movement of another agent (e.g., an amoeba) within the environment. 
The environment considered in both cases is shown in Fig.~\ref{fig:FRenvCaseIandII}.
%

\subsection{Static Environment Compression}\label{subsec:results_StaticEnv}

In this case, the transmitting agent has access to the occupancy map shown in Fig.~\ref{fig:originalTrueEnv}, which is assumed to be unchanging.
With knowledge of the environment, the sending agent is to employ our framework to assist another (receiving) agent with developing a map estimate in the presence of communication constraints.
To this end, we will assume that the receiving agent has no apriori knowledge of the environment, and thus its initial map estimate $\hat X_0$ is an occupancy grid with the value of $0.5$ for every cell.
We run our framework for $N = 11$ steps with a communication constraint sequence $\{n_k\}_{k=1}^{11}$ that corresponds to $\{1\%, 5\%, 2\%, 20\%, 30\%, 5\%, 8\%, 15\%, 40\%, 15\%, 10\%\}$ of the finest-resolution nodes.
It is important to note that although the environment in this case does not evolve as a function of time, the sender is faced with communication resource constraints and therefore cannot transmit all map information to the receiver at once.
As a result, this scenario will demonstrate how the framework we developed in this paper enables the sender to incrementally send map information to the receiver to generate its map estimate, which is shown to converge to the true map as time progresses.
%

\begin{figure}[t]
	\centering
	\subfloat[]{\label{fig:staticEnvDistortionBar}\includegraphics[width=0.49\columnwidth]{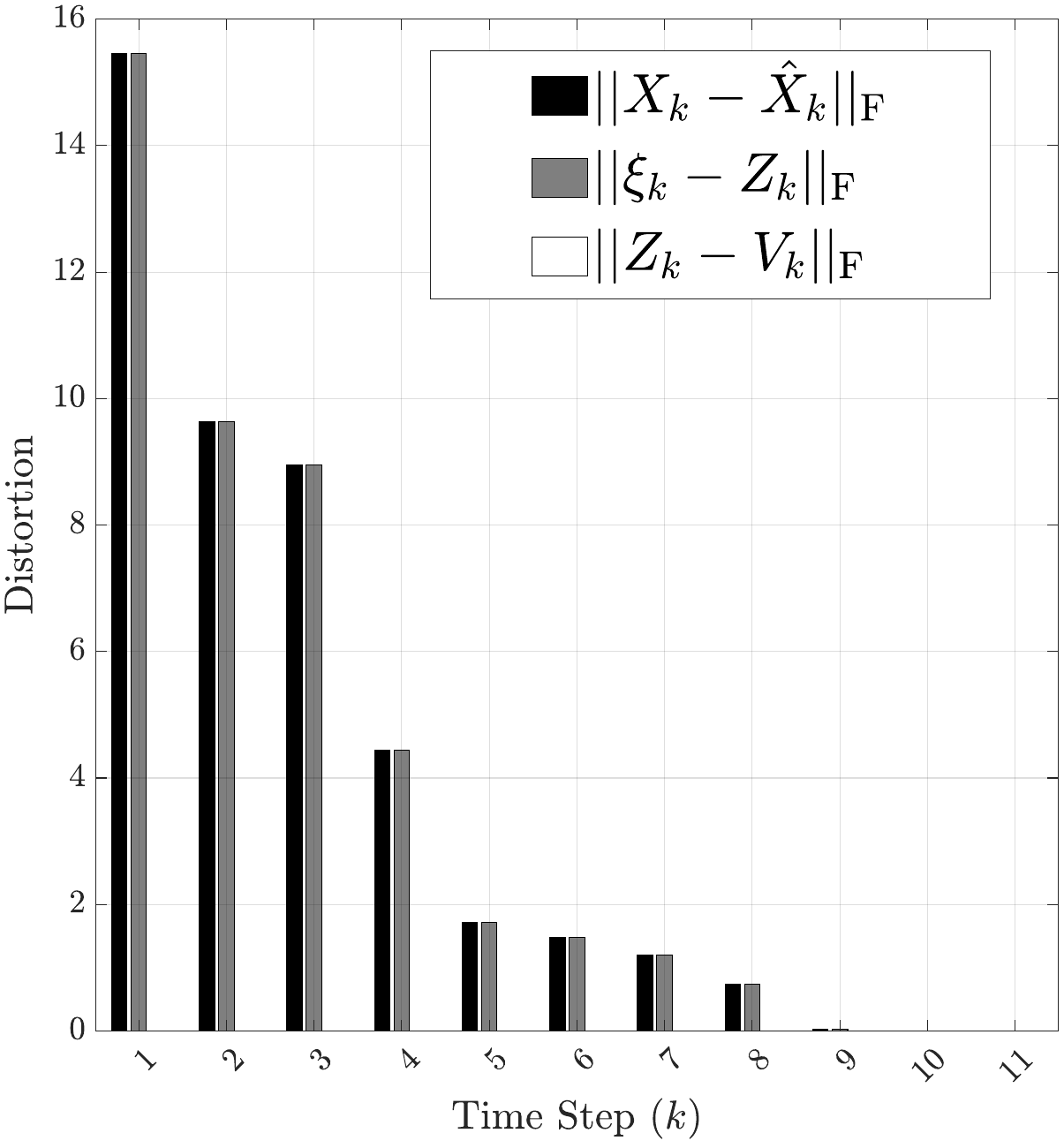}}
	\hfill
	\subfloat[]{\label{fig:staticEnvLeafsBar}\includegraphics[width=0.49\columnwidth]{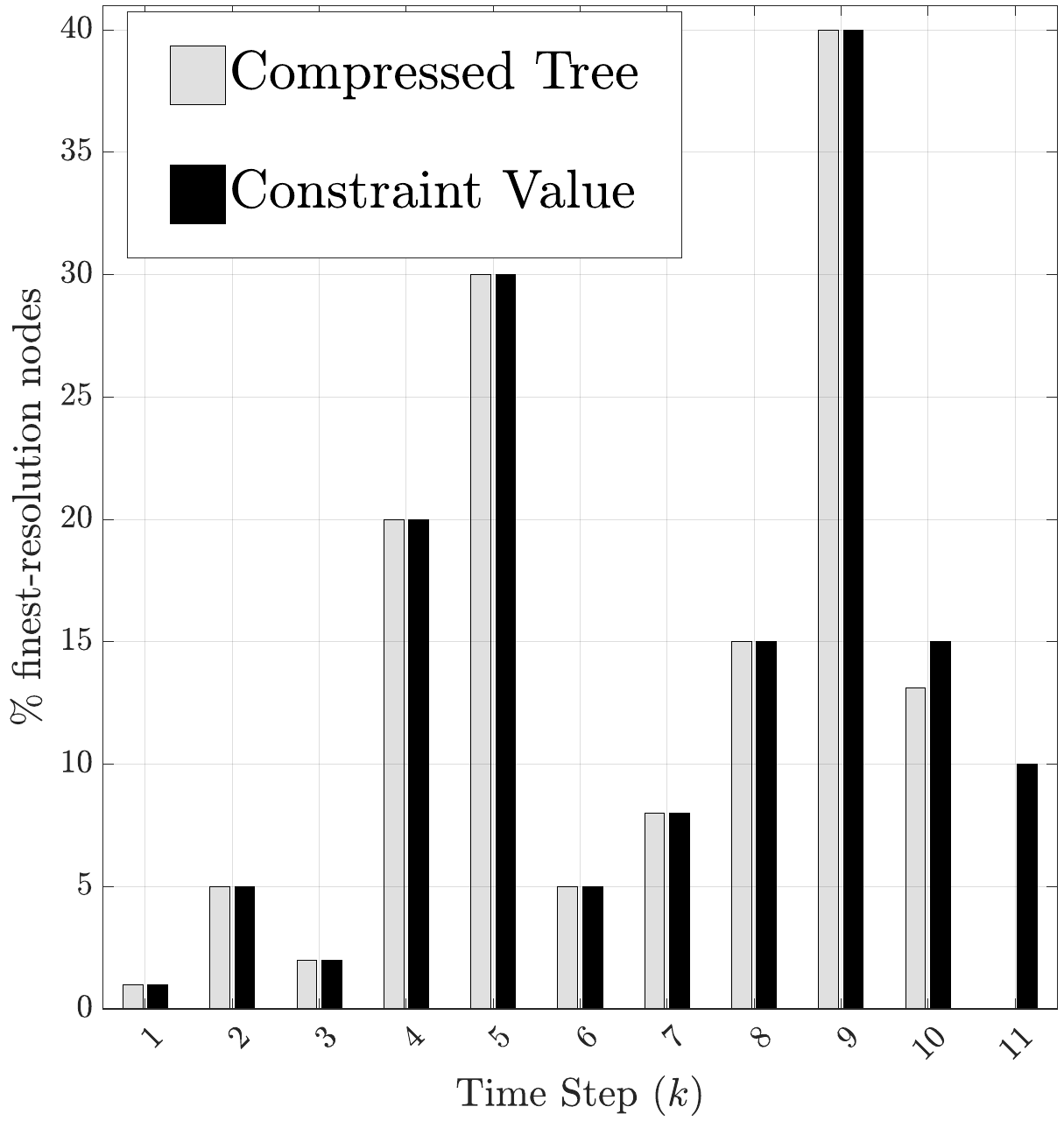}} 
	\caption{Distortion and number of leaf nodes (as percentage of finest-resolution map) vs. time-step. (a) Distortion vs. time-step. (b) Number of leaf nodes of the tree encoder solution vs. time-step with the constraint $b_k$ shown.}
	\label{fig:staticEnvDistortionLeafsBar}
\end{figure}

\begin{figure}[b]
	 \centering
	 \subfloat[Time step $k=1$.]{\label{fig:staticEnvEstimate1}\includegraphics[width=0.33\columnwidth]{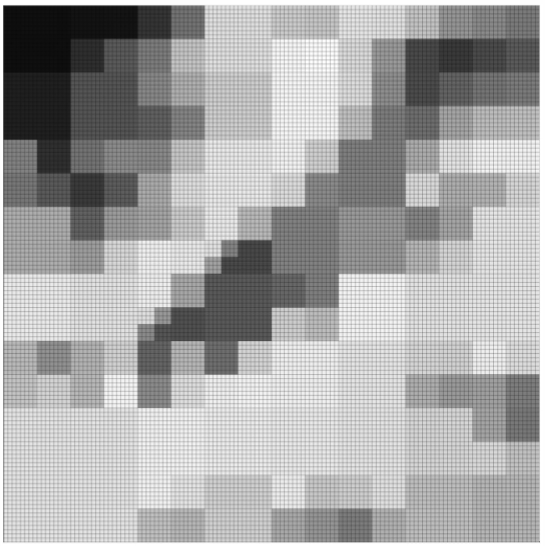}}
	 \hfill
	 \subfloat[Time step $k=4$.]{\label{fig:staticEnvEstimate2}\includegraphics[width=0.33\columnwidth]{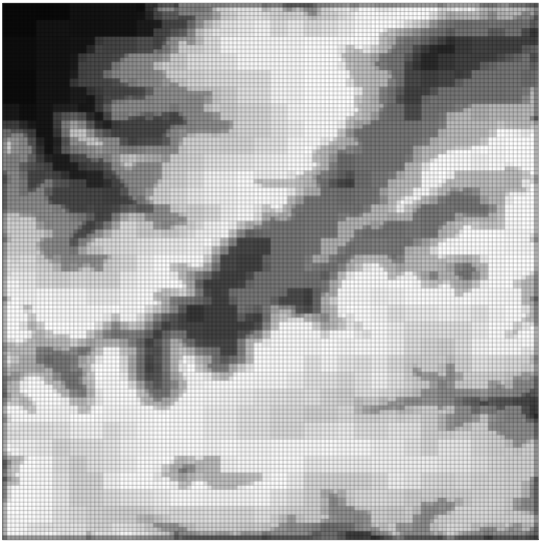}}
	 \hfill
	 \subfloat[Time step $k=10$.]{\label{fig:staticEnvEstimate3}\includegraphics[width=0.33\columnwidth]{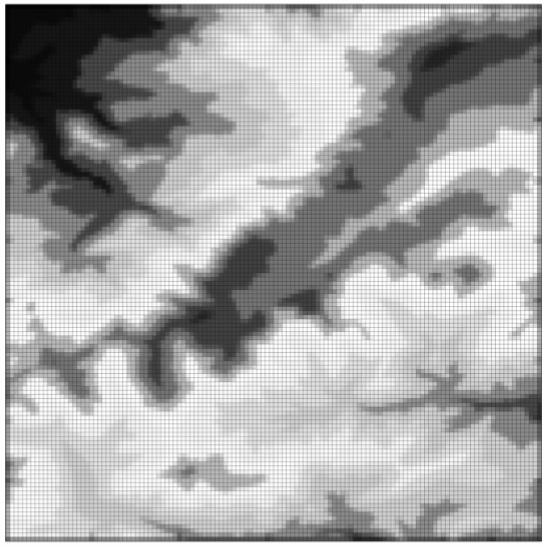}} 
	 \caption{Estimated map sequence $\hat X_k$.}
	 \label{fig:staticEnvEstimatedMapSequence}
\end{figure}

%
In Fig.~\ref{fig:staticEnvLeafsBar} we show the communication constraint together with the number of cells in the compressed representation for each time step as determined by Algorithm~\ref{alg:mapCompressionAlg}.
A sample of map estimates $\hat X_k$ that are obtained from our framework are shown in Fig.~\ref{fig:staticEnvEstimatedMapSequence}.
Furthermore, we show in Fig.~\ref{fig:staticEnvDistortionBar} the map distortion as a function of time-step, where $V_k = \min\{\max\{-\hat X_{k-1},Z_k\},1-\hat X_{k-1}\}$ so that $\lVert Z_k - V_k \rVert_{\Frob}$ is the second term in~\eqref{eq:objTScompProb}.
Consequently, Fig.~\ref{fig:staticEnvDistortionBar} shows the contribution of each of the distortion terms, the innovation-distortion $\lVert \xi_k - Z_k\rVert_{\Frob}$ and the decoding distortion $\lVert Z_k - V_k \rVert_{\Frob}$, to the upper bound of the overall estimation error $\lVert X_k - \hat X_{k} \rVert_{\Frob}$.
We also see from Fig.~\ref{fig:staticEnvDistortionBar} that the distortion $\lVert X_k - \hat X_k \rVert_{\Frob}$ is monotone decreasing with the number of time-steps, which is expected since the transmitting agent will incrementally encode map information not already held by the receiver as time progresses as the environment is static.
Notice also from Figs.~\ref{fig:staticEnvLeafsBar} and~\ref{fig:staticEnvEstimatedMapSequence} that the receiver is able to form a reasonable estimate of the environment despite a strict communication constraint placed on the sender. 
%

\begin{figure}[tbh]
	\centering
	\subfloat[Time step $k=1$.]{\label{fig:innovationSignal1}\includegraphics[width=0.33\columnwidth]{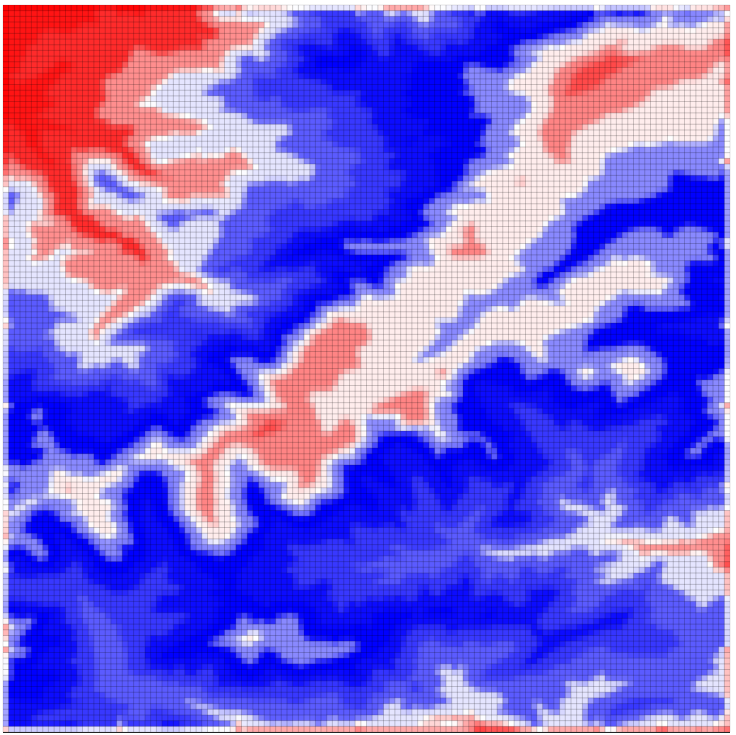}}
	\hfill
	\subfloat[Time step $k=4$.]{\label{fig:innovationSignal2}\includegraphics[width=0.33\columnwidth]{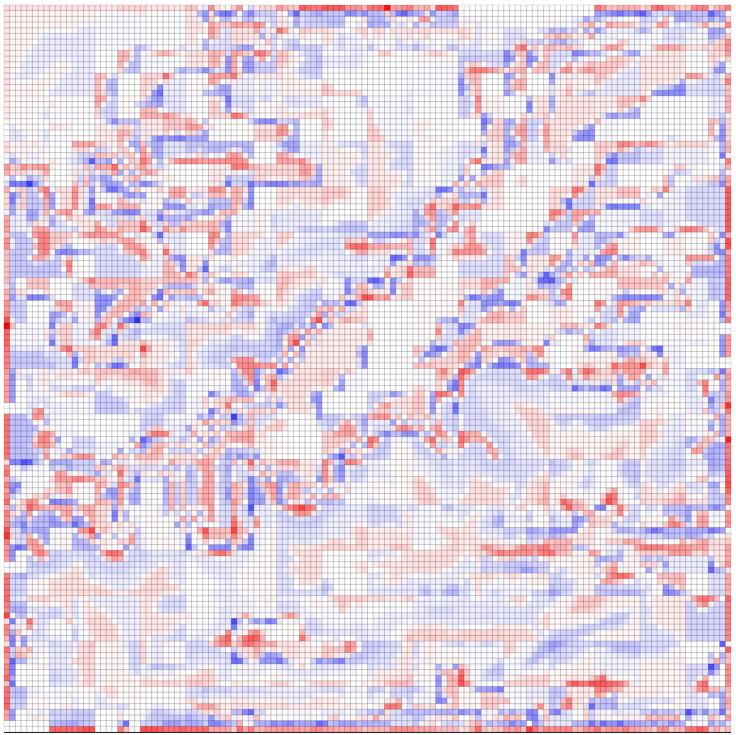}}
	\hfill
	\subfloat[Time step $k=10$.]{\label{fig:innovationSignal3}\includegraphics[width=0.33\columnwidth]{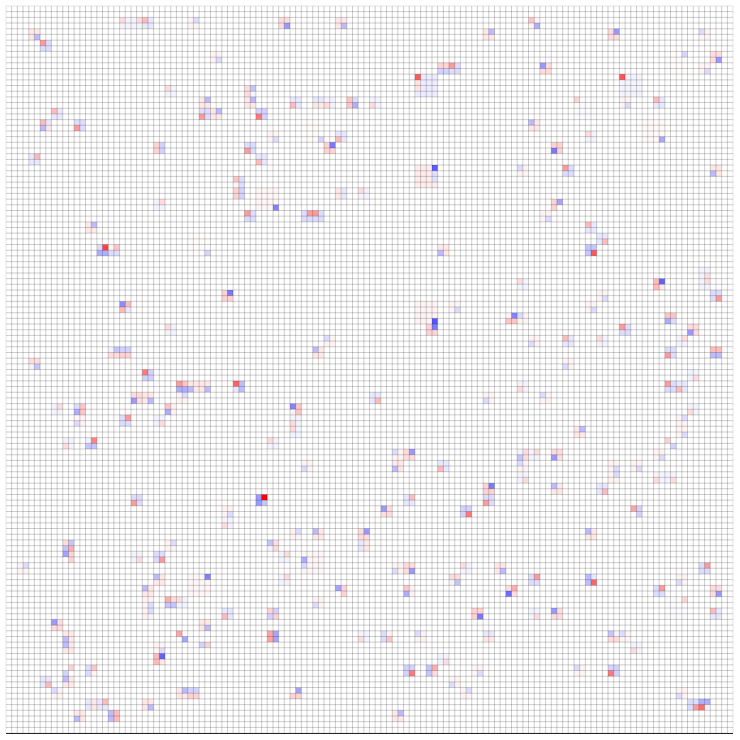}} \\
	%
	\subfloat[Time step $k=1$.]{\label{fig:innovationSignalOvrly1}\includegraphics[width=0.33\columnwidth]{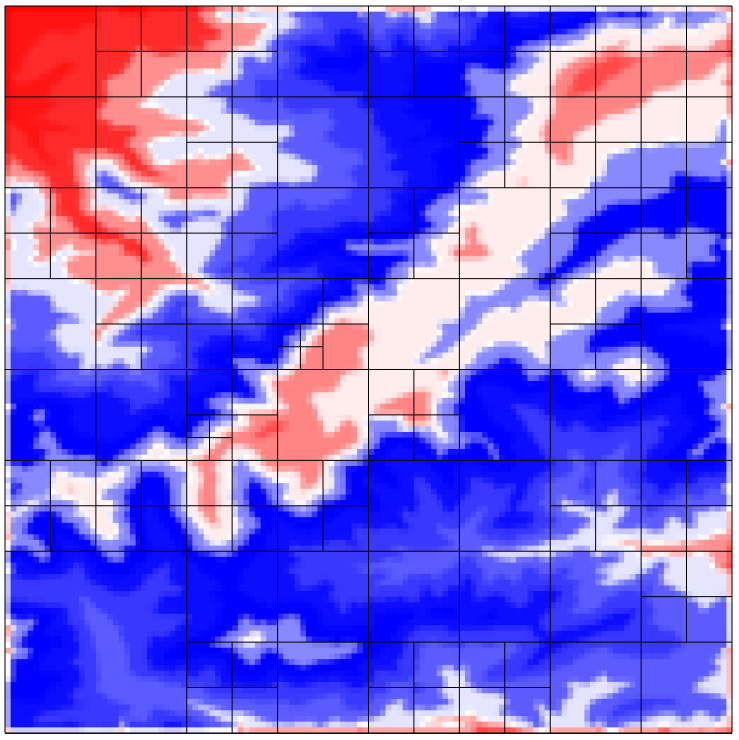}}
	\hfill
	\subfloat[Time step $k=4$.]{\label{fig:innovationSignalOvrly2}\includegraphics[width=0.33\columnwidth]{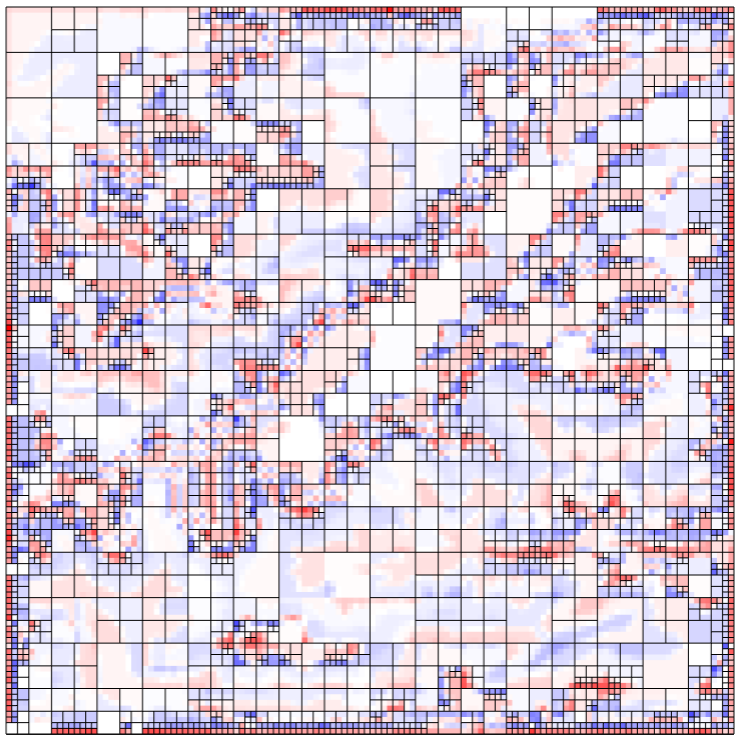}}
	\hfill
	\subfloat[Time step $k=10$.]{\label{fig:innovationSignalOvrly3}\includegraphics[width=0.33\columnwidth]{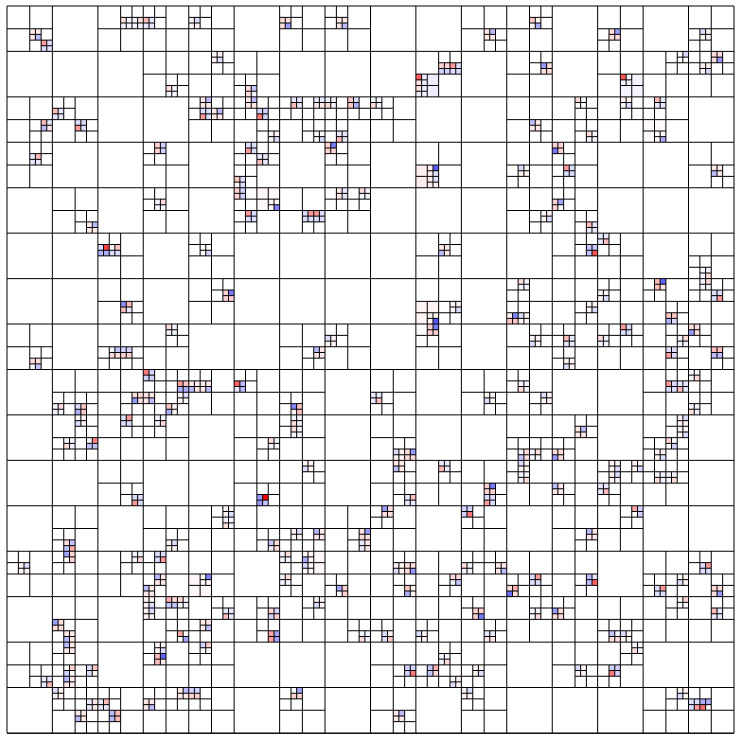}} 
	\caption{Map innovation signal together with the selected hierarchical tree encoder. (a)-(c) map innovation signal $\xi_k$. (d)-(f) The map innovation signal in (a)-(c) with the tree encoder overlay. Color corresponds to innovation signal sign and magnitude: red is a cell with positive innovation signal, blue is a cell with negative innovation signal, and white is a cell with an innovation signal of zero. Intensity of corresponding color indicates magnitude of signal strength.}
	\label{fig:innovationSeqTemplateStatic}
\end{figure}

The information not already held by the receiver is captured by the map innovation signal, shown in Fig.~\ref{fig:innovationSeqTemplateStatic}, alongside the tree-structured encoder for selected time-steps.
From Fig.~\ref{fig:innovationSignal1} we see that, at the initial time step $k=1$, the map innovation signal is generally non-zero, which indicates that the receiving agent does not hold the correct understanding of the environment.
Those regions of the map that do have zero innovation at $k=1$ correspond to areas that have a (true) probability of occupancy equal to $0.5$, as these values coincide with the occupancy probability of the initial map estimate $\hat X_0$.
We also see from the sequence shown in Figs.~\ref{fig:innovationSignal1}-\ref{fig:innovationSignal3} that, as time progresses, the map innovation is reduced and eventually nearly zero for all cells in Fig.~\ref{fig:innovationSignal3}.
This is consistent with Fig.~\ref{fig:staticEnvEstimatedMapSequence} in that, as the map estimate $\hat X_k$ approaches the true map $X_k$, the discrepancy between the two maps reduces, and therefore the map innovation signal will approach zero.
An additional observation from Figs.~\ref{fig:innovationSignalOvrly1} and~\ref{fig:innovationSignalOvrly3} is that the multi-resolution hierarchical tree encoder maintains high-resolution in regions where the map innovation discrepancy is large (in absolute value).
It should be noted, however, that the granularity of the tree-structured encoder depicted in Figs.~\ref{fig:innovationSignalOvrly1}-\ref{fig:innovationSignalOvrly3} is restricted by the leaf node constraint shown in Fig.~\ref{fig:staticEnvLeafsBar}, the latter of which represents the communication constraints imposed on our problem.
Therefore, the tree solution to~\eqref{eq:encObj}-\eqref{eq:encCons3} will attempt to resolve as many areas of large map innovation distortion as possible with high resolution, subject to the granularity constraints imposed on the problem.
On this note, we also observe from Fig.~\ref{fig:staticEnvLeafsBar} that it need not be the case that our framework utilizes the full communication bandwidth to send map information.
Instead, our framework works to resolve as many non-zero map innovation signal regions as possible by maintaining high resolution in those areas, while ensuring that other regions with zero map innovation maintain aggregated to reduce communication redundancy.
%

\subsection{Dynamic (time-changing) Environments}\label{subsec:results_DynEnv}

We will now consider the scenario of compressing a time-evolving environment $\{X_k\}_{k=1}^{N}$ due to the presence of a dynamic obstacle or object (e.g., a vehicle, human, or other agent), where the dynamic obstacle is assumed to be a circular amoeba that follows the path shown in Fig.~\ref{fig:fullEnvWithPath}.
The path followed by the amoeba is generated via application of Dijkstra's algorithm and employing the cost function discussed in~\cite{larsson2021information_b} with the parameters $\lambda_1 = 0$ and $\lambda_2 = 1$ and $\varepsilon = 0.5$.
It is important to note that the path shown in Fig.~\ref{fig:fullEnvWithPath} is not known to the sending agent, who will only know about the movement of the amoeba from its own observations.
The time-evolving nature of the environment is caused by the movement of the amoeba, which changes the occupancy of the cells it occupies.
%

\begin{figure}[t]
	\centering
	\subfloat[]{\label{fig:dynEnvDistortionBar}\includegraphics[width=0.49\columnwidth]{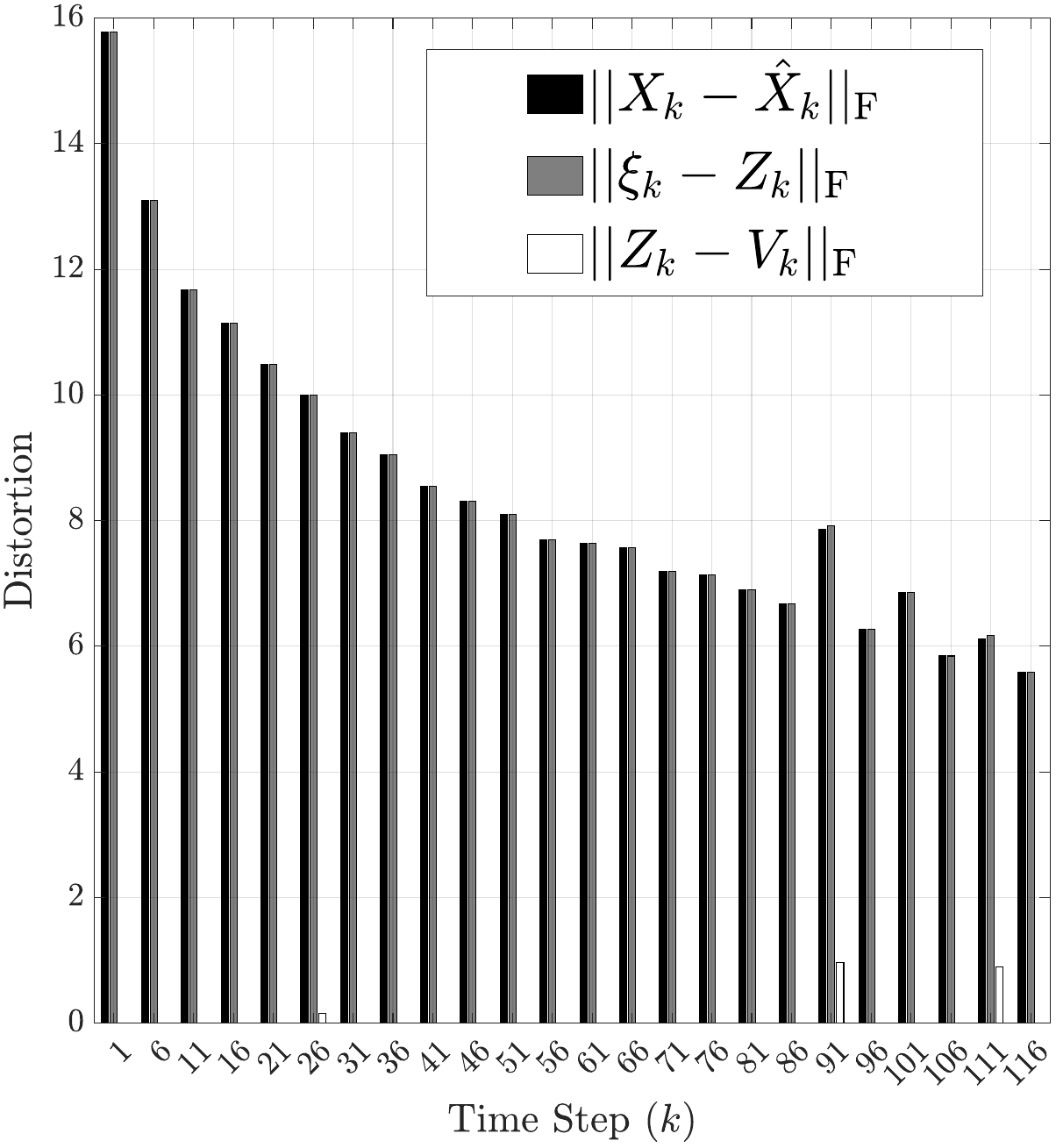}}
	\hfill
	\subfloat[]{\label{fig:dynEnvLeafsBar}\includegraphics[width=0.49\columnwidth]{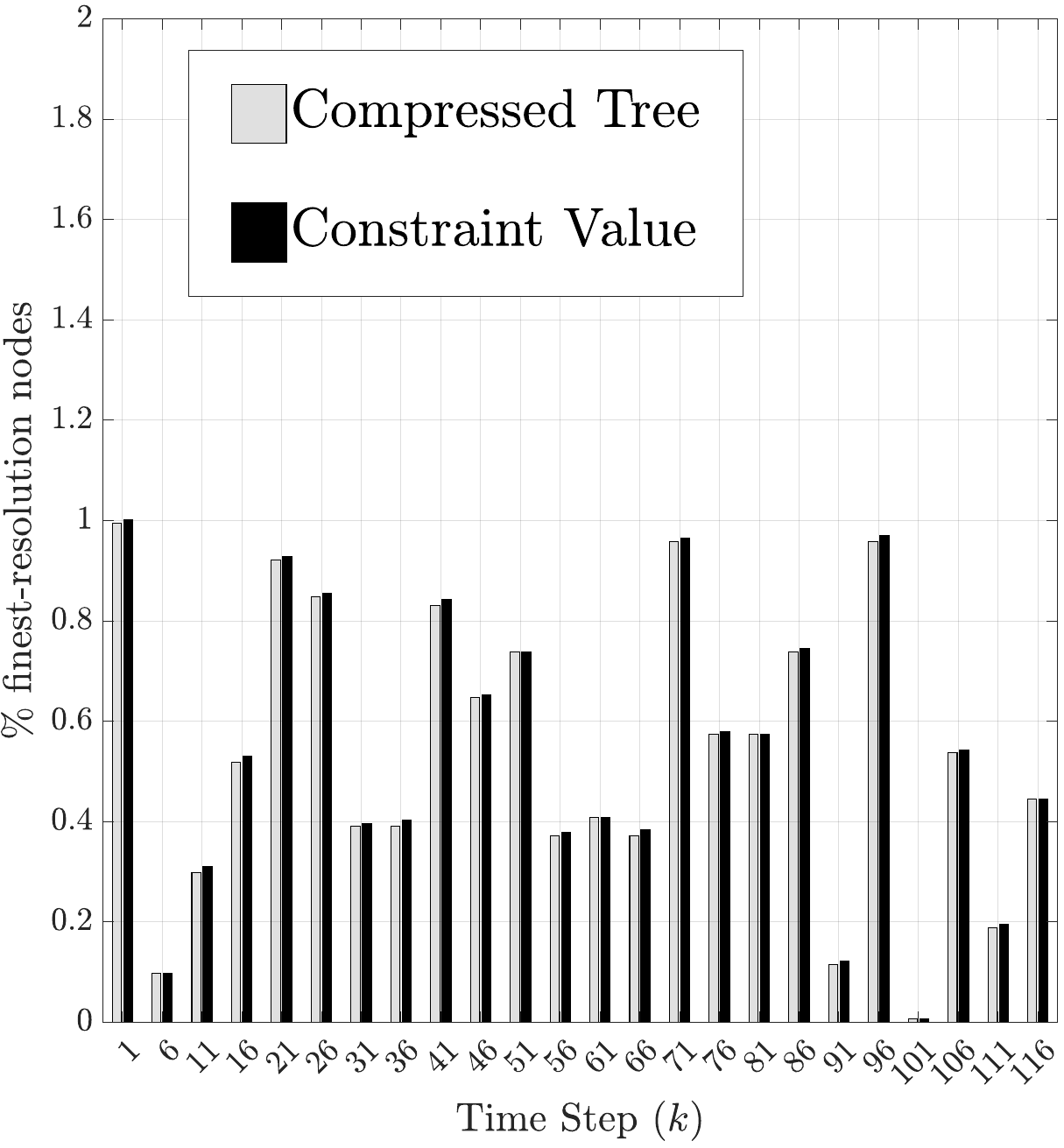}} 
	\caption{Distortion and number of leaf nodes (as percentage of finest-resolution map) vs. time-step for selected time steps. (a) Distortion vs. time-step. (b) Number of leaf nodes of the tree encoder solution vs. time-step with the constraint $b_k$ shown.}
	\label{fig:dynEnvDistortionLeafsBar}
\end{figure}

%
Application of our framework to the aforementioned scenario results with the distortion sequence shown in Fig~\ref{fig:dynEnvDistortionBar}.
Comparing with the results from Section~\ref{subsec:results_StaticEnv}, we observe from Fig.~\ref{fig:dynEnvDistortionBar} that, in contrast to Fig.~\ref{fig:staticEnvDistortionBar}, when the map is dynamic, the distortion $\lVert X_k - \hat X_k \rVert_{\Frob}$ is not monotone decreasing.
This occurs since in the case of a dynamic environment, the sending agent need not be faced with the same environment at each time-step, and thus monotone information transfer is not necessarily possible as regions may appear that increases the information-discrepancy between sender and receiver. 
Thus, some map regions may lack correct receiver information due to coarse updates, incorrect beliefs, or both.
%

\begin{figure}[tbh]
	\centering
	\subfloat[Time step $k=1$.]{\label{fig:dynEnvTrue1}\includegraphics[width=0.33\columnwidth]{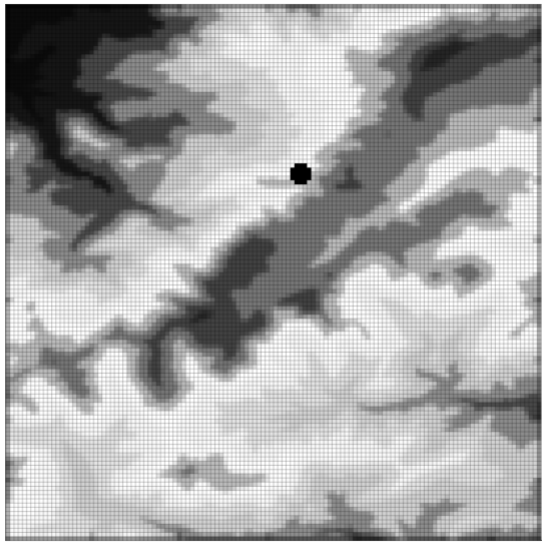}}
	\hfill
	\subfloat[Time step $k=91$.]{\label{fig:dynEnvTrue2.pdf}\includegraphics[width=0.33\columnwidth]{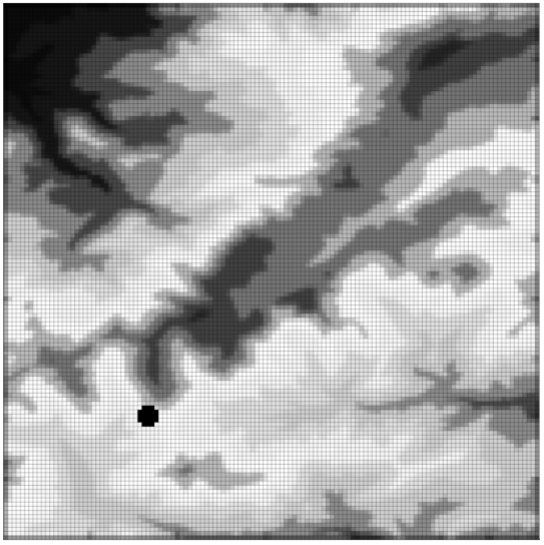}}
	\hfill
	\subfloat[Time step $k=116$.]{\label{fig:dynEnvTrue3}\includegraphics[width=0.33\columnwidth]{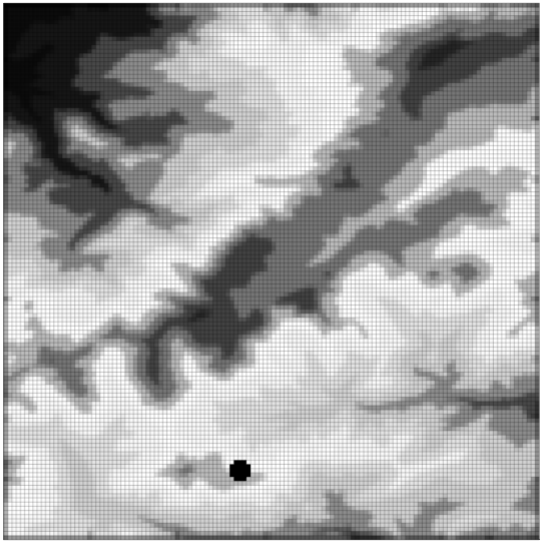}} \\
	\subfloat[Time step $k=1$.]{\label{fig:dynEnvEstimate1}\includegraphics[width=0.33\columnwidth]{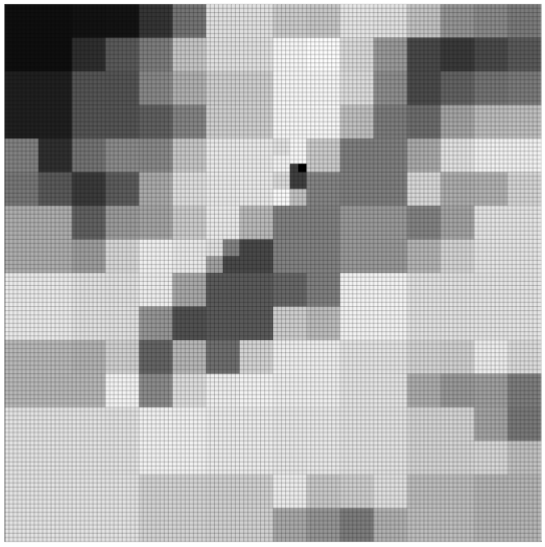}}
	\hfill
	\subfloat[Time step $k=91$.]{\label{fig:dynEnvEstimate2}\includegraphics[width=0.33\columnwidth]{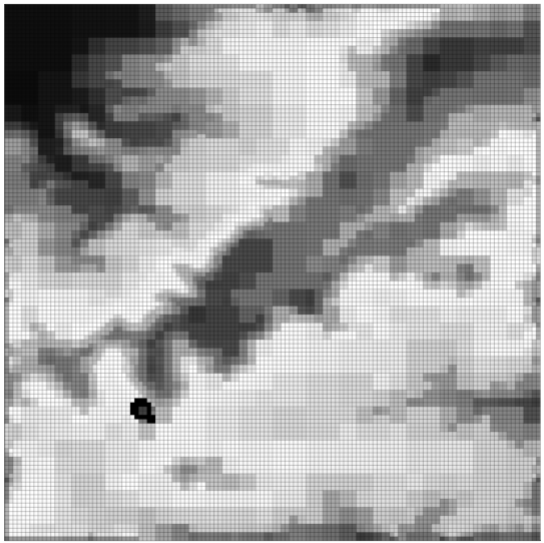}}
	\hfill
	\subfloat[Time step $k=116$.]{\label{fig:dynEnvEstimate3}\includegraphics[width=0.33\columnwidth]{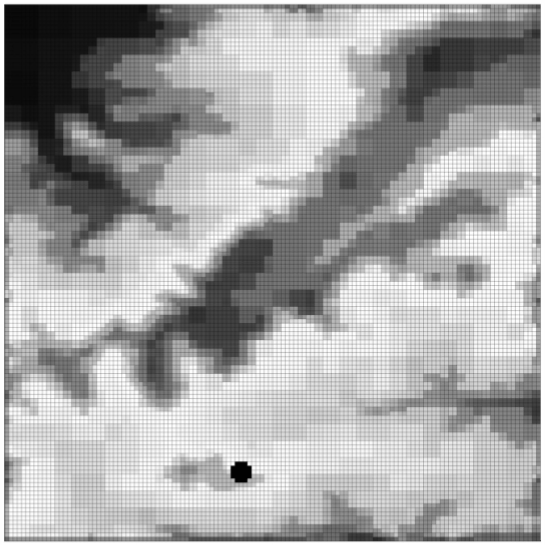}} 
	\caption{Sender-perceived map sequence $X_k$ and receiver-estimated map sequence $\hat X_k$ as a function of time step $k$. (a)-(c) Sender-perceived map $X_k$. (d)-(f) Receiver estimated map $\hat X_k$.}
	\label{fig:dynEnvEstimatedMapSequence}
\end{figure}

%
A sequence of map estimates obtained from our framework is shown in Fig.~\ref{fig:dynEnvEstimatedMapSequence}.
From Fig.~\ref{fig:dynEnvEstimatedMapSequence} we observe that the location of the amoeba is well-resolved in the map estimate sequence, despite the fact that the number of leaf nodes permitted at each time step is constrained to be no more than approximately $1\%$ of the total number of finest-resolution cells, as seen in Fig~\ref{fig:dynEnvLeafsBar}.
Moreover, we see from the sequence of maps in Figs.~\ref{fig:dynEnvEstimate1}-\ref{fig:dynEnvEstimate3} that the initial time-step our framework works to quantize the map information broadly, since after only one time-step the receiver holds a reasonably representative estimate of the map in Fig.~\ref{fig:dynEnvEstimate1}.
We also see from the sequence in Figs.~\ref{fig:dynEnvEstimate1}-\ref{fig:dynEnvEstimate3} that once our framework has communicated map information broadly, the approach works to provide good tracking resolution of the moving amoeba, resulting with the amoeba being clearly discernible at the final time-step.
The above observations are further corroborated by Fig.~\ref{fig:dynEnvinnovationSeqTemplate}, which shows the sequence of map innovation signals and selected encoder.
From Figs.~\ref{fig:dynEnvinnovationSignalOvrly2} and~\ref{fig:dynEnvinnovationSignalOvrly3} we see that the intensity of the map innovation signal is relatively small in regions outside of the moving amoeba, and thus the encoder selected at later time-steps is one that maintains very high resolution around only the regions of the map changing due to the moving amoeba.
%

\begin{figure}[tbh]
	\centering
	\subfloat[Time step $k=1$.]{\label{fig:dynEnvinnovationSignal1}\includegraphics[width=0.33\columnwidth]{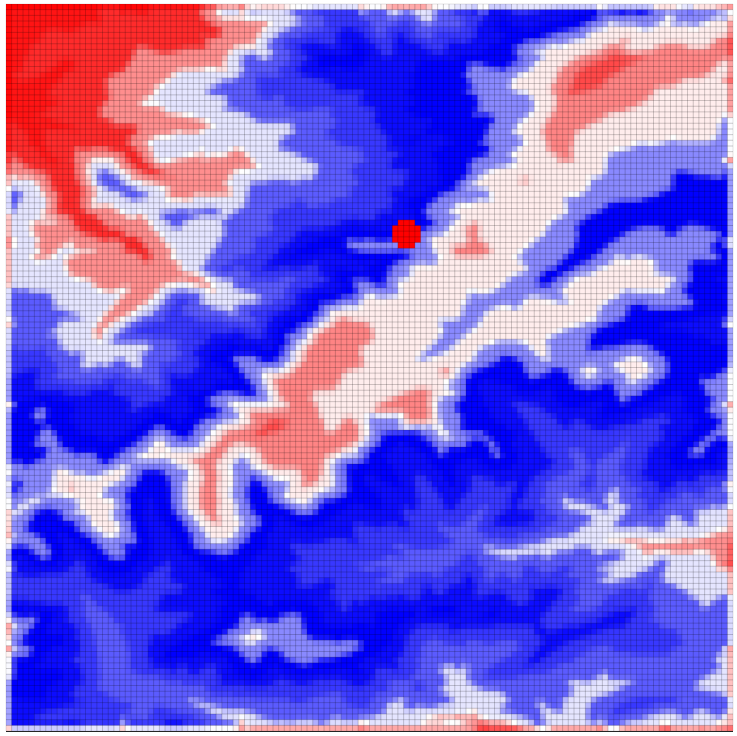}}
	\hfill
	\subfloat[Time step $k=91$.]{\label{fig:dynEnvinnovationSignal2}\includegraphics[width=0.33\columnwidth]{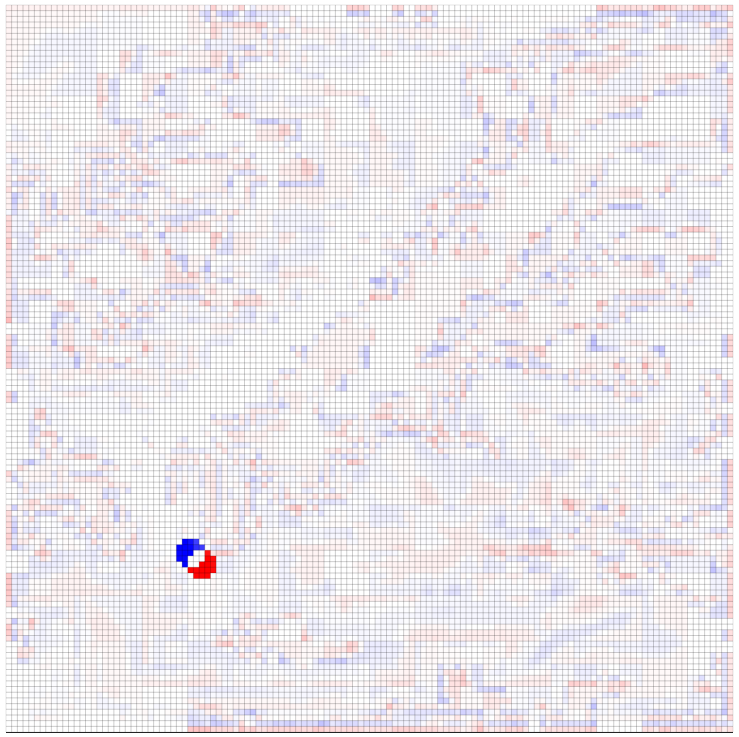}}
	\hfill
	\subfloat[Time step $k=116$.]{\label{fig:dynEnvinnovationSignal3}\includegraphics[width=0.33\columnwidth]{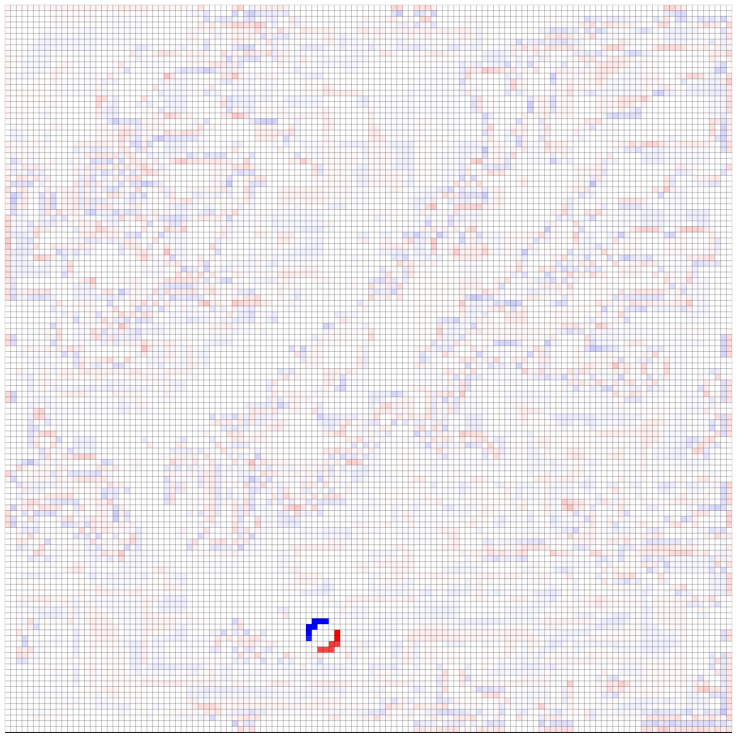}} \\
	%
	\subfloat[Time step $k=1$.]{\label{fig:dynEnvinnovationSignalOvrly1}\includegraphics[width=0.33\columnwidth]{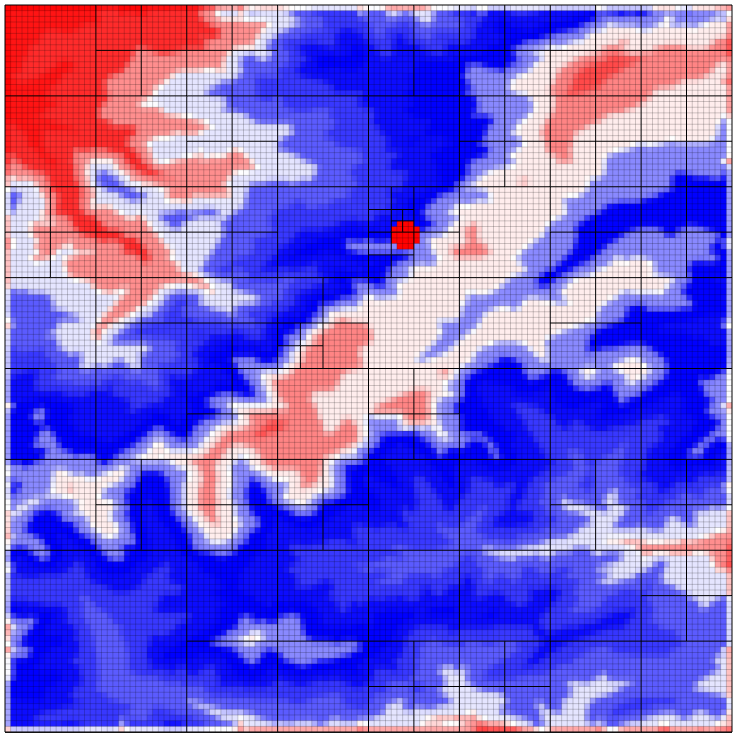}}
	\hfill
	\subfloat[Time step $k=91$.]{\label{fig:dynEnvinnovationSignalOvrly2}\includegraphics[width=0.33\columnwidth]{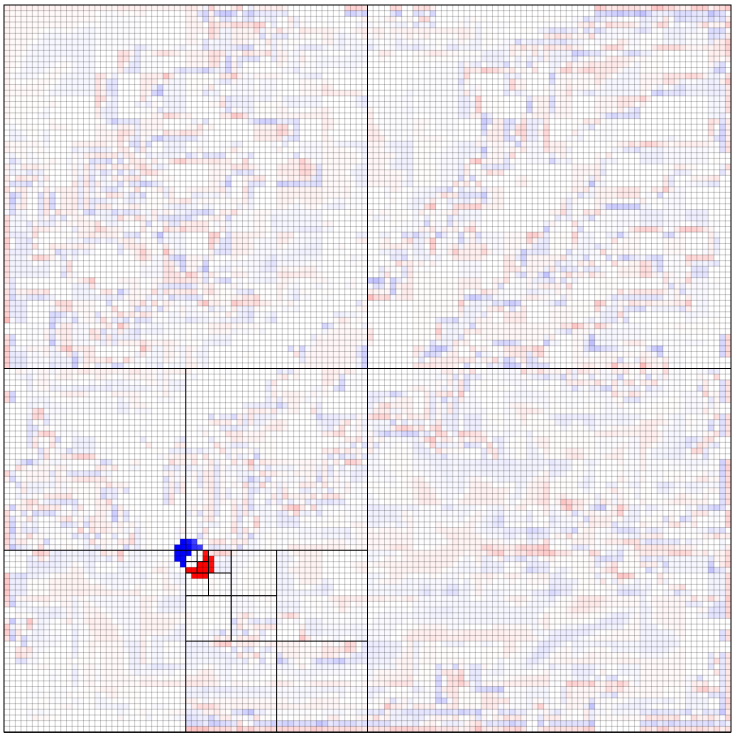}}
	\hfill
	\subfloat[Time step $k=116$.]{\label{fig:dynEnvinnovationSignalOvrly3}\includegraphics[width=0.33\columnwidth]{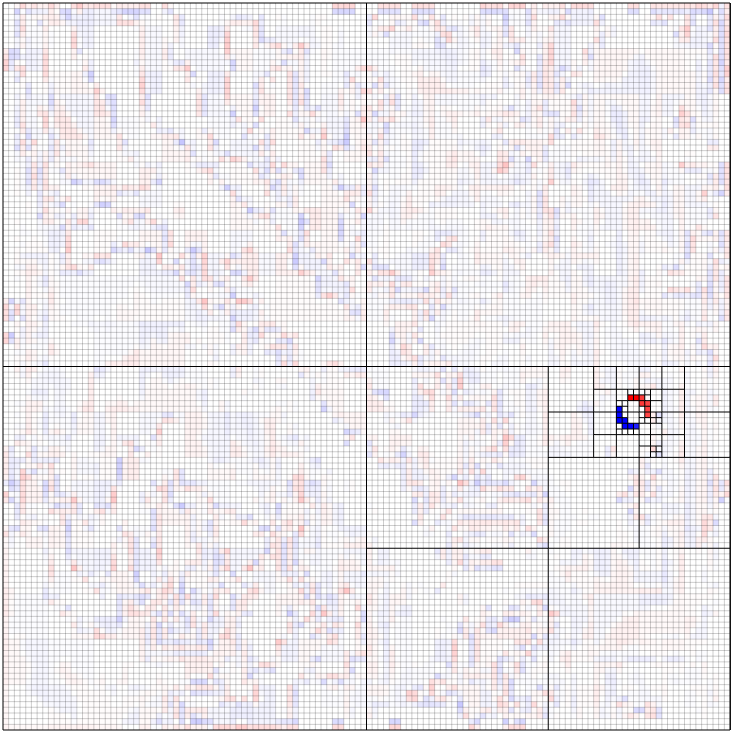}} 
	\caption{Map innovation signals together with the selected hierarchical tree encoder. (a)-(c) map innovation signal $\xi_k$. (d)-(f) The map innovation signal in (a)-(c) with the tree encoder overlay. Color corresponds to innovation signal sign and magnitude: red is a cell with positive innovation signal, blue is a cell with negative innovation signal, and white is a cell with an innovation signal of zero. Intensity of corresponding color indicates magnitude of signal strength.}
	\label{fig:dynEnvinnovationSeqTemplate}
\end{figure}

\section{Conclusions}\label{sec:conclusion}

In this paper, we developed a framework for occupancy grid compression in dynamic environments. 
Our approach leveraged concepts from signal compression theory to formulate a map compression problem that aims to minimize the distortion between the true and the receiver-reconstructed occupancy map as a function of time while satisfying communication-bandwidth constraints. 
To solve our problem, we introduce a \emph{map innovation signal}, which quantifies the information discrepancy between sending and receiving agents, and constrain the signal encoder to correspond to a hierarchical multi-resolution (quad)tree.
We showcase the utility of our framework by compressing large probabilistic occupancy grids in both static (non-time varying) and dynamic (time-varying) cases.
Our results show that a significant degree of map compression can be accomplished while relaying large amounts of map information.

\bibliographystyle{IEEEtran}


\begin{thebibliography}{10}
	\providecommand{\url}[1]{#1}
	\csname url@samestyle\endcsname
	\providecommand{\newblock}{\relax}
	\providecommand{\bibinfo}[2]{#2}
	\providecommand{\BIBentrySTDinterwordspacing}{\spaceskip=0pt\relax}
	\providecommand{\BIBentryALTinterwordstretchfactor}{4}
	\providecommand{\BIBentryALTinterwordspacing}{\spaceskip=\fontdimen2\font plus
		\BIBentryALTinterwordstretchfactor\fontdimen3\font minus
		\fontdimen4\font\relax}
	\providecommand{\BIBforeignlanguage}[2]{{%
			\expandafter\ifx\csname l@#1\endcsname\relax
			\typeout{** WARNING: IEEEtran.bst: No hyphenation pattern has been}%
			\typeout{** loaded for the language `#1'. Using the pattern for}%
			\typeout{** the default language instead.}%
			\else
			\language=\csname l@#1\endcsname
			\fi
			#2}}
	\providecommand{\BIBdecl}{\relax}
	\BIBdecl
	
	\bibitem{psomiadis2024communication}
	E.~Psomiadis, D.~Maity, and P.~Tsiotras, ``Communication-aware map compression
	for online path-planning,'' in \emph{2024 IEEE International Conference on
		Robotics and Automation (ICRA)}.\hskip 1em plus 0.5em minus 0.4em\relax IEEE,
	2024, pp. 12\,368--12\,374.
	
	\bibitem{Cowlagi2012}
	R.~V. Cowlagi and P.~Tsiotras, ``Multiresolution motion planning for autonomous
	agents via wavelet-based cell decompositions,'' \emph{IEEE Transactions on
		Systems, Man, and Cybernetics, Part B (Cybernetics)}, vol.~42, no.~5, pp.
	1455--1469, 2012.
	
	\bibitem{Cowlagi2008}
	------, ``Multiresolution path planning with wavelets: A local replanning
	approach,'' in \emph{American Control Conference}, 2008, pp. 1220--1225.
	
	\bibitem{Cowlagi2007}
	------, ``Beyond quadtrees: Cell decompositions for path planning using wavelet
	transforms,'' in \emph{IEEE Conference on Decision and Control}, 2007, pp.
	1392--1397.
	
	\bibitem{kraetzschmar2004probabilistic}
	G.~K. Kraetzschmar, G.~P. Gassull, and K.~Uhl, ``Probabilistic quadtrees for
	variable-resolution mapping of large environments,'' \emph{IFAC Proceedings
		Volumes}, vol.~37, no.~8, pp. 675--680, 2004.
	
	\bibitem{hauer2015multi}
	F.~Hauer, A.~Kundu, J.~M. Rehg, and P.~Tsiotras, ``Multi-scale perception and
	path planning on probabilistic obstacle maps,'' in \emph{2015 IEEE
		International Conference on Robotics and Automation (ICRA)}.\hskip 1em plus
	0.5em minus 0.4em\relax IEEE, 2015, pp. 4210--4215.
	
	\bibitem{hauer2016reduced}
	F.~Hauer and P.~Tsiotras, ``Reduced complexity multi-scale path-planning on
	probabilistic maps,'' in \emph{2016 IEEE International Conference on Robotics
		and Automation (ICRA)}.\hskip 1em plus 0.5em minus 0.4em\relax IEEE, 2016,
	pp. 83--88.
	
	\bibitem{Tsiotras2007}
	P.~Tsiotras and E.~Bakolas, ``A hierarchical on-line path planning scheme using
	wavelets,'' in \emph{2007 European Control Conference (ECC)}, 2007, pp.
	2806--2812.
	
	\bibitem{tsiotras_multiresolution_2012}
	\BIBentryALTinterwordspacing
	P.~Tsiotras, D.~Jung, and E.~Bakolas, ``Multiresolution {Hierarchical}
	{Path}-{Planning} for {Small} {UAVs} {Using} {Wavelet} {Decompositions},''
	\emph{Journal of Intelligent \& Robotic Systems}, vol.~66, no.~4, pp.
	505--522, Jun. 2012. [Online]. Available:
	\url{https://doi.org/10.1007/s10846-011-9631-z}
	\BIBentrySTDinterwordspacing
	
	\bibitem{bakolas2008multiresolution}
	E.~Bakolas and P.~Tsiotras, ``Multiresolution path planning via sector
	decompositions compatible to on-board sensor data,'' in \emph{AIAA Guidance,
		Navigation and Control Conference and Exhibit}, 2008, p. 7238.
	
	\bibitem{larsson2020q}
	D.~T. Larsson, D.~Maity, and P.~Tsiotras, ``Q-tree search: An
	information-theoretic approach toward hierarchical abstractions for agents
	with computational limitations,'' \emph{IEEE Transactions on Robotics},
	vol.~36, no.~6, pp. 1669--1685, 2020.
	
	\bibitem{tishby2000IBmethod}
	N.~Tishby, F.~C. Pereira, and W.~Bialek, ``The information bottleneck method,''
	in \emph{The 37th annual Allerton Conference on Communication}, September
	1999, p. 368–377.
	
	\bibitem{thomas2006elements}
	M.~Thomas and A.~T. Joy, \emph{Elements of information theory}.\hskip 1em plus
	0.5em minus 0.4em\relax Wiley-Interscience, 2006.
	
	\bibitem{larsson2021information}
	D.~T. Larsson, D.~Maity, and P.~Tsiotras, ``Information-theoretic abstractions
	for resource-constrained agents via mixed-integer linear programming,'' in
	\emph{Proceedings of the Workshop on Computation-Aware Algorithmic Design for
		Cyber-Physical Systems}, 2021, pp. 1--6.
	
	\bibitem{larsson2023linear}
	------, ``A linear programming approach for resource-aware
	information-theoretic tree abstractions,'' in \emph{Computation-Aware
		Algorithmic Design for Cyber-Physical Systems}.\hskip 1em plus 0.5em minus
	0.4em\relax Springer, 2023, pp. 101--138.
	
	\bibitem{larsson2022generalized}
	------, ``A generalized information-theoretic framework for the emergence of
	hierarchical abstractions in resource-limited systems,'' \emph{Entropy},
	vol.~24, no.~6, p. 809, 2022.
	
	\bibitem{Nelson2015}
	E.~Nelson and N.~Michael, ``Information-theoretic occupancy grid compression
	for high-speed information-based exploration,'' in \emph{2015 IEEE/RSJ
		International Conference on Intelligent Robots and Systems (IROS)}, 2015, pp.
	4976--4982.
	
	\bibitem{larsson2021information_b}
	D.~T. Larsson, D.~Maity, and P.~Tsiotras, ``Information-theoretic abstractions
	for planning in agents with computational constraints,'' \emph{IEEE Robotics
		and Automation Letters}, vol.~6, no.~4, pp. 7651--7658, 2021.
	
	\bibitem{psomiadis2024multi}
	E.~Psomiadis, D.~Maity, and P.~Tsiotras, ``Multi-agent task-driven exploration
	via intelligent map compression and sharing,'' \emph{arXiv preprint
		arXiv:2403.14780}, 2024.
	
	\bibitem{borkar1997lqg}
	V.~S. Borkar and S.~K. Mitter, ``Lqg control with communication constraints,''
	in \emph{Communications, computation, control, and signal processing: a
		tribute to Thomas Kailath}.\hskip 1em plus 0.5em minus 0.4em\relax Springer,
	1997, pp. 365--373.
	
	\bibitem{yuksel2019note}
	S.~Y\"{u}ksel, ``A note on the separation of optimal quantization and control
	policies in networked control,'' \emph{SIAM Journal on Control and
		Optimization}, vol.~57, no.~1, pp. 773--782, 2019.
	
	\bibitem{maity2023optimal}
	D.~Maity and P.~Tsiotras, ``Optimal quantizer scheduling and controller
	synthesis for partially observable linear systems,'' \emph{SIAM Journal on
		Control and Optimization}, vol.~61, no.~4, pp. 2682--2707, 2023.
	
\end{thebibliography}


\appendix
For the simplicity of the presentation we suppress the time indices.
Let $\hat X^- $ denote the  map estimate of the previous time. Then,
\begin{align*}
	\lVert X - f_d(f_e(X)) \rVert_F & =\lVert X  - \hat X^- + \hat X^- - f_d(f_e(X))  \rVert_F, \\
    & = \lVert \xi + \hat X^- - f_d(f_e(X)) \rVert_F,
\end{align*}
and thus~\eqref{eq:mainOrigObjectiveFunc}-\eqref{eq:mainOrigConstraint2} is equivalent to the problem
\begin{equation}\label{eq:encRexpressOPTIntermediate1}
	\min_{f_e,\bar f_d} ~ \lVert \xi - \bar f_d(f_e(X)) \rVert_F,
\end{equation}
\begin{align}
   \text{subject to} \hspace{2 cm} B(f_e(X)) &\leq b, \hspace{1cm}\\
	0 \leq \hat X^- + \bar f_d(f_e(X)) &\leq 1, \label{eq:modifiedConstraint2} 
\end{align}
where $\bar f_d(f_e(X)) = f_d(f_e(X)) - \hat X^-$ is the function $f_d(f_e(X))$ shifted by the previous map estimate.

Now, note that the innovation map $\xi$ is available at the encoder. 
Therefore, for any encoding function $g_e$ that operates on the innovation map, we may upper-bound the objective function \eqref{eq:encRexpressOPTIntermediate1}~as 
\begin{align*}
    \lVert \xi - \bar f_d(f_e(X)) \rVert_F \le \lVert \xi - g_e(\xi) \rVert_F + \lVert g_e(\xi) - \bar f_d(f_e(X)) \rVert_F,
\end{align*}
where $\lVert \xi - g_e(\xi) \rVert_F$ denotes the \textit{loss} due to encoding and $\lVert g_e(\xi) - \bar f_d(f_e(X)) \rVert_F$ denotes the loss from decoding. 
Consequently,
\begin{align*}
    \min_{f_e,\bar f_d}\, \lVert \xi -\! \bar f_d(f_e(X)) \rVert_F \le 
    \min_{f_e,\bar f_d, g_e} \, \lVert \xi - g_e(\xi) \rVert_F + \lVert g_e(\xi) -\! \bar f_d(f_e(X)) \rVert_F.
\end{align*}
Likewise, for any decoder $g_d$ that operates on the encoded innovation map, we may further upper bound the objective according to
\begin{align*}
    \min_{f_e,\bar f_d}\, \lVert \xi -\! \bar f_d(f_e(X)) \rVert_F \le 
    \min_{f_e,\bar f_d, g_e, g_d} \!\! &\lVert \xi - g_e(\xi) \rVert_F + \lVert g_e(\xi) -  g_d(g_e(\xi))\rVert_F \\
    &~~~~~ +\lVert g_d(g_e(\xi)) - \bar f_d(f_e(X)) \rVert_F.
\end{align*}
The choice $f_e(X) = g_e(\xi)$ and $\bar f_d (\cdot) = g_d(\cdot)$ minimizes the upper bound by making the last term zero.
Subsequently, by plugging in these choices for $f_e$ and $\bar f_d$, we arrive at the following upper bound minimization problem: 
\begin{align*}
    \min_{g_e, g_d}\,\lVert \xi - g_e(\xi) \rVert_F + \lVert g_e(\xi) -  g_d(g_e(\xi))\rVert_F
\end{align*}
\begin{align*}
   \text{subject to}\hspace{2cm} B(g_e(\xi)) &\leq b,\\
	0 \leq \hat X^- + g_d(g_e(X)) &\leq 1,
\end{align*}
which completes the proof. \qed
\end{document}